\documentclass[conference]{IEEEtran}
\IEEEoverridecommandlockouts

\usepackage{cite}
\usepackage{amsmath,amssymb,amsfonts,amsthm}
\usepackage{algorithmic}
\usepackage{graphicx}
\usepackage{textcomp}
\usepackage{xcolor}

\usepackage[hyphens]{url}      
\usepackage[hidelinks]{hyperref}

\usepackage[caption=false,font=footnotesize]{subfig}

\usepackage[small]{caption}
\usepackage{booktabs}
\usepackage{algorithm}

\usepackage{multirow}

\def\BibTeX{{\rm B\kern-.05em{\sc i\kern-.025em b}\kern-.08em
    T\kern-.1667em\lower.7ex\hbox{E}\kern-.125emX}}

\newtheorem{theorem}{Theorem}

\newtheorem{lemma}[theorem]{Lemma}

\begin{document}

\title{TSGM: Regular and Irregular Time-series Generation using Score-based Generative Models
}

\author{
\IEEEauthorblockN{Haksoo Lim}
\IEEEauthorblockA{
\textit{KAIST}\\
Seoul, Korea, Republic of\\
limhaksoo96@kaist.ac.kr}
\and
\IEEEauthorblockN{Jaehoon Lee}
\IEEEauthorblockA{\textit{LG AI Research}\\
Seoul, Korea, Republic of\\
jaehoon.lee@lgresearch.ai}
\and
\IEEEauthorblockN{Sewon Park}
\IEEEauthorblockA{\textit{Samsung SDS}\\
Seoul, Korea, Republic of\\
sw0413.park@samsung.com}
\and
\IEEEauthorblockN{Minjung Kim}
\IEEEauthorblockA{\textit{Samsung SDS}\\
Seoul, Korea, Republic of\\
mj100.kim@samsung.com}
\and
\IEEEauthorblockN{Noseong Park}
\IEEEauthorblockA{\textit{KAIST}\\
Daejeon, Korea, Republic of\\
noseong@kaist.ac.kr}
}

\maketitle

\begin{abstract}
Score-based generative models (SGMs) have demonstrated unparalleled sampling quality and diversity in numerous fields, such as image generation, voice synthesis, and tabular data synthesis, etc. Inspired by those outstanding results, we apply SGMs to synthesize time-series by learning its conditional score function. To this end, we present a conditional score network for time-series synthesis, deriving a denoising score matching loss tailored for our purposes. In particular, our presented denoising score matching loss is the conditional denoising score matching loss for time-series synthesis. In addition, our framework is such flexible that both regular and irregular time-series can be synthesized with minimal changes to our model design. Finally, we obtain exceptional synthesis performance on various time-series datasets, achieving state-of-the-art sampling diversity and quality.
\end{abstract}


\section{Introduction}
Time-series frequently occurs in our daily lives, e.g., stock data, climate data, and so on. Especially, time-series forecasting and classification are popular research topics in the field of machine learning~\cite{ahmed2010timeforecast,fu2011time}. 
In many cases, however, time-series samples are incomplete and/or the number of samples is insufficient, in which case training machine learning models cannot be fulfilled in a robust way. To overcome the limitation, time-series synthesis has been studied actively recently~\cite{https://doi.org/10.48550/arxiv.1806.07366,10.1007/978-3-030-59137-3_34}. These synthesis models have been designed in various ways, including variational autoencoders (VAEs) and generative adversarial networks (GANs)~\cite{desai2021timevae,yoon2019timegan,https://doi.org/10.48550/arxiv.2210.02040}. 

Moreover, real-world time series often inevitably contain missing values because of privacy reasons or medical events. It has been noted that the missing values involve important information, called \textit{informative missingness}~\cite{a6ce2314-039c-31b2-80e5-de4ad323600d}. To remedy the problem, several works have been developed to deal with irregular time-series, i.e., the inter-arrival time between observations is not fixed and/or some observations can be missing~\cite{Schafer2002MissingDO,Che2016RecurrentNN,DBLP:journals/corr/abs-2005-08926}, in which case synthesizing irregular time series is challenging~\cite{https://doi.org/10.48550/arxiv.2210.02040}.

Score-based generative models (SGMs) have shown good sampling quality and diversity in numerous fields, such as image generation, voice synthesis, and tabular data synthesis, etc~\cite{yang2023diffusion}. However, in time-series generation, SGMs should consider \textit{autoregressiveness}, meaning each time-series observation is generated in consideration of its previously generated observations which makes time-series generation more difficult~\cite{yoon2019timegan,https://doi.org/10.48550/arxiv.2210.02040}. 
To this end, we propose the method of \underline{\textbf{T}}ime-series generation using conditional \underline{\textbf{S}}core-based \underline{\textbf{G}}enerative \underline{\textbf{M}}odel (TSGM), which consists of three neural networks, i.e., an encoder, a score network, and a decoder (see Figure~\ref{fig1}). 

\begin{table}
        \footnotesize
        \centering
        \setlength{\linewidth}{\textwidth}
        \begin{tabular}{c|cc|cc|cc}
        \hline
            \multirow{3}{*}{Method} & \multicolumn{6}{c}{Olympic Rankings} \\ \cline{2-7}
            & \multicolumn{2}{c|}{Gold} & \multicolumn{2}{c|}{Silver} & \multicolumn{2}{c}{Bronze} \\ \cline{2-7}
             & Regular & Irregular & R  & I & R & I \\ \hline
        TSGM-VP & 3 & 9 & 3 & 9 & 1 & 4 \\ 
        TSGM-subVP & 5 & 14 & 1 & 6 & 1 & 3 \\ \hline
        TimeGAN & 0 & 0 & 0 & 0 & 1 & 0 \\ 
        TimeVAE & 0 & 0 & 0 & 0 & 1 & 1 \\ 
        GT-GAN & 0 & 1 & 0 & 2 & 2 & 4 \\
        KoVAE & 2 & 4 & 1 & 8 & 2 & 9 \\
        Diffusion-TS & 5 & 1 & 1 & 0 & 2 & 0 \\ \hline
        \end{tabular}
\caption{The table illustrates how many medals each method gets across all datasets and evaluation metrics, based on the generation evaluation scores presented in Table~\ref{table2_1} and Table~\ref{table10}. 
Our method with the two specific types, TSGM-VP and TSGM-subVP, achieves superior generation performance compared to baselines.}
\label{tbl:summary}
\end{table}

\begin{table}
\centering
\includegraphics[width=1\linewidth]{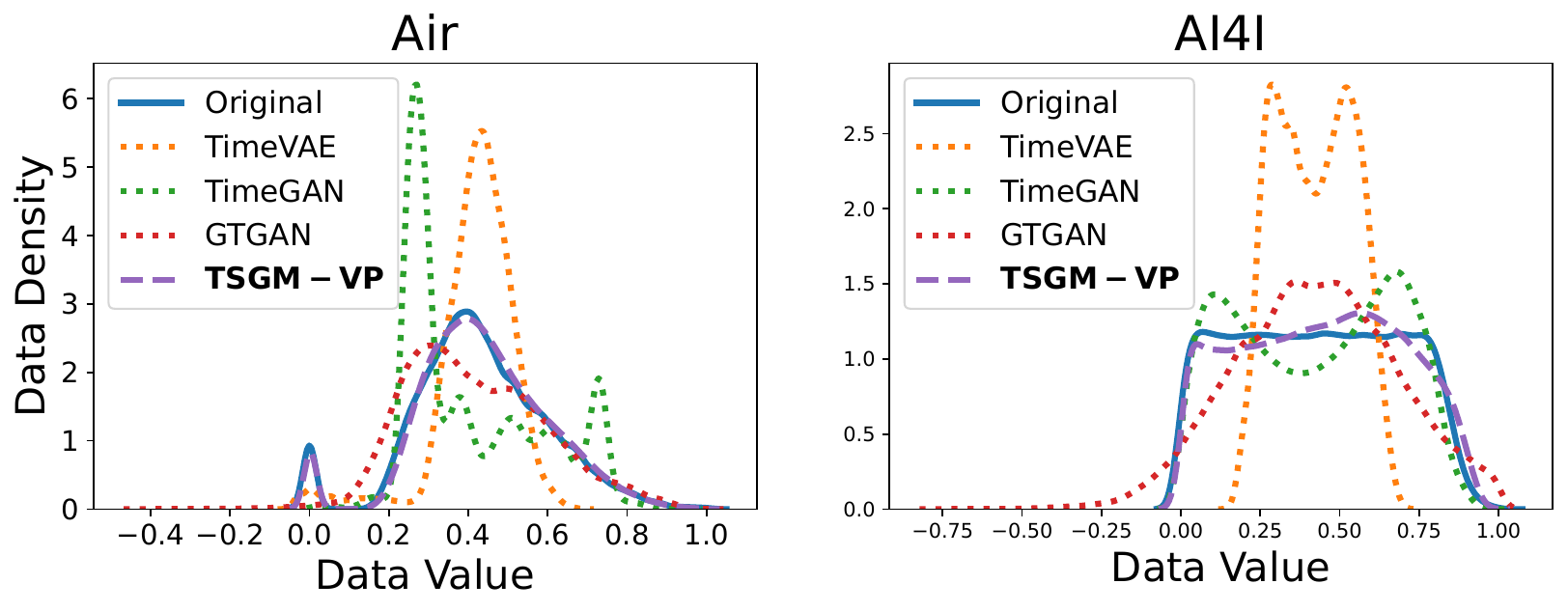}
\captionof{figure}{The KDE plots show the estimated distributions of original data and ones generated by several methods in the Air and AI4I datasets  --- we ignore time stamps for drawing these distributions. Unlike baseline methods, the distribution of TSGM-VP is almost identical to the original one. These figures provide an evidence of the excellent generation quality and diversity of our method. For TSGM-subVP, similar results are observed.}
\label{fig:kde}
\end{table}

\paragraph{Score-based time-series synthesis} 
SGMs have its potential to deal with autoregressiveness~\cite{DBLP:journals/corr/abs-2010-12810}. We design our own autoregressive denoising score matching loss for time-series generation and prove its correctness (see Section~\ref{autosgm}). Along with its performances from image generation domain (cf. Fig.~\ref{fig:kde}), here is our main motivation of using SGMs: 

\textit{The mathematical structure of SGMs naturally aligns with the sequential nature of time-series data}.

Besides, we design a conditional score network on time-series, which learns the gradient of the conditional log-likelihood.

\paragraph{Regular vs. irregular time-series synthesis} 
Time-series are prevalent throughout our daily lives. Although regular time-series are easy to be considered, it is common to have missing values because of privacy issue~\cite{Che2016RecurrentNN,DBLP:journals/corr/abs-2005-08926}, which makes irregularly sampled time-series\footnote{For example, Physionet~\cite{PhysioNet}, a famous dataset for time series classification, deliberately removed 90\% of observations to protect the privacy of patients, posing challenges for learning and analysis. The synthesized time series can be used instead.}. This hinders time-series processing and we consider the hardest problem condition: \textit{providing complete time-series given not only regular time-series, but also irregular ones}. To this end, our method is considered `universial' in that both both regular and irregular time-series samples can be treated with minimal changes to our model design. For synthesizing regular time series, we use a recurrent neural network-based encoder and decoder. Continuous-time methods, such as neural controlled differential equations~\cite{DBLP:journals/corr/abs-2005-08926} and GRU-ODE~\cite{DBLP:journals/corr/abs-1905-12374}, can be used as our encoder and decoder for synthesizing irregular time series (see Section~\ref{sec:enc}).

We conduct in-depth experiments with 4 real-world datasets under regular and irregular settings. To be specific, for the irregular settings, we randomly drop 30\%, 50\%, and 70\% of observations from regular time-series on training, then generate complete time-series on evaluation. Therefore, we test with 16 different settings, i.e., 4 datasets for one regular and three irregular settings. Our specific choices of 11 baselines include almost all existing types of time-series generative paradigms, ranging from VAEs to GANs and diffusion model. In Table~\ref{tbl:summary} and Figure~\ref{fig:kde}, we compare our method to the baselines, ranking methods by their evaluation scores and estimating data distribution by kernel density estimation (KDE). We also visualize real and generated time-series samples onto a latent space using t-SNE~\cite{maaten2008tsne} in Figure~\ref{fig2}. Our proposed method shows the best generation quality in most of the cases. Furthermore, the t-SNE and KDE visualization results provide intuitive evidence that our method's generation diversity is also superior to that of the baselines. Our contributions are summarized as follows:

\begin{enumerate}
    \item We, for the first time, propose an SGM-based universal time-series synthesis method.
    \item We derive our own denoising score matching loss considering the autoregressive nature of sequential data, connecting SGMs to time-series generation domain.
    \item We conduct comprehensive experiments with 4 real-world datasets and 11 baselines under one regular and three irregular settings since our method supports both regular and irregular time-series. Overall, our proposed method shows the best generation quality and diversity.
\end{enumerate}

\section{Related Work and Preliminaries}

\subsection{Score-based Generative Models}\label{sec:SGM}
SGMs offer several advantages over other generative models, including their higher generation quality and diversity. SGMs follow a two-step process, wherein i) gaussian noises are continuously added to a sample and ii) then removed to recover a new sample. These processes are known as the forward and reverse processes, respectively. In this section, we provide a brief overview of the original SGMs in~\cite{song2021SDE}, which will be adapted for the time-series generation tasks.

\subsubsection{Forward and Reverse Process}

At first, SGMs add noises with the following stochastic differential equation (SDE):
\begin{equation*}
    d\textbf{x}^s=\textbf{f}(s,\textbf{x}^s)ds+g(s)d\textbf{w},  s \in [0,1],
\end{equation*}
where $\textbf{w} \in \mathbb{R}^{\dim(\textbf{x})}$ is a multi-dimensional Brownian motion, $\textbf{f}(s,\cdot): \mathbb{R}^{\dim(\textbf{x})} \rightarrow \mathbb{R}^{\dim(\textbf{x})}$ is a vector-valued drift term, and $g : [0,1] \rightarrow \mathbb{R}$ is a scalar-valued diffusion function. Hereafter, we define $\textbf{x}^s$ as a noisy sample diffused at time $s \in [0,1]$ from an original sample $\textbf{x} \in \mathbb{R}^{\dim(\textbf{x})}$. Therefore, $\textbf{x}^s$ can be understood as a stochastic process following the SDE.

There are several options for $\textbf{f}$ and $g$: variance exploding(VE), variance preserving(VP), and subVP. It is proved that VE and VP are continuous generalizations of the two discrete diffusion methods~\cite{song2021SDE}: NCSN and DDPM in~\cite{song2019smld,sohl2015ddpm}.
Also, the subVP method shows, in general, better negative log-likelihood (NLL)~\cite{song2021SDE}. We describe the exact form of each SDE in Table~\ref{tbl:driftanddiffusion}.

Note that we only use the subVP-based TSGM in our main experiments and exclude the VE and VP-based one for its inferiority for time series synthesis in our experiments, but checked the VP-based method from ablation study for its better performances than that of the VE.

SGMs run the forward SDE with a sufficiently large number of steps to make sure that the diffused sample converges to a Gaussian distribution at the final step. The score network $M_{\theta}(s,\textbf{x}^s)$ learns the gradient of the log-likelihood $\nabla_{\textbf{x}^s} \log p(\textbf{x}^s)$, which will be used in the reverse process.

For the forward SDE, there exists the following corresponding reverse SDE~\cite{anderson1982reverse}:
\begin{equation*}
    d\textbf{x}^s=[\textbf{f}(s,\textbf{x}^s)-g^2(s)\nabla_{\textbf{x}^s}{\log p(\textbf{x}^s)}]ds+g(s)d\bar{\textbf{w}}.
\end{equation*}

The formula suggests that if knowing the score function, $\nabla_{\textbf{x}^s}{\log p(\textbf{x}^s)}$, we can recover real samples from the prior distribution $p_1(\textbf{x}) \sim \mathcal{N}(\mu, \sigma^2)$, where $\mu, \sigma$ vary depending on the forward SDE type.

\begin{table}
\small
\centering
\begin{tabular}{c|c|c}
    \toprule
    SDE & drift ($\textbf{f}$) & diffusion ($g$) \\
    \midrule
     VE    & 0 & $\sqrt{\frac{d\sigma^2(s)}{ds}}$ \\
     VP & $-\frac{1}{2}\beta(s)\textbf{x}^s$ & $\sqrt{\beta(s)}$ \\    
     subVP  & $-\frac{1}{2}\beta(s)\textbf{x}^s$ & $\sqrt{\beta(s)(1-e^{-2\int_0^s\beta(t)dt})}$ \\
    \bottomrule
\end{tabular}
\caption{Comparison of drift and diffusion terms. $\sigma(s)$ means positive noise values which are increasing, and $\beta(s)$ denotes noise values in [0,1], which are used in SMLD~\cite{song2019smld} and DDPM~\cite{ho2020ddpm}.}\label{tbl:driftanddiffusion}
\end{table}

\subsubsection{Training and Sampling}

In order for the model $M$ to learn the score function, the model has to optimize the following loss function:
\begin{equation*}
L(\theta) = \mathbb{E}_{s}\{\lambda(s)\mathbb{E}_{\textbf{x}^s}[\left\|M_{\theta}(s,\textbf{x}^s)-{\nabla}_{\textbf{x}^s}\log p(\textbf{x}^s)\right\|_2^2]\},
\end{equation*}where $s$ is uniformly sampled over $[0,1]$ with an appropriate weight function $\lambda(s):[0,1]\rightarrow \mathbb{R}$. However, using the above formula is computationally prohibitive~\cite{JMLR:v6:hyvarinen05a,DBLP:journals/corr/abs-1905-07088}. It was found that the loss can be substituted with the following denoising score matching loss~\cite{vincent2011matching}:
\begin{equation*}
L^*(\theta)=\mathbb{E}_{s}\{\lambda(s)\mathbb{E}_{\textbf{x}^0}\mathbb{E}_{\textbf{x}^s|\textbf{x}^0}[\left\|M_{\theta}(s,\textbf{x}^s)-{\nabla}_{\textbf{x}^s}\log p(\textbf{x}^s|\textbf{x}^0)\right\|_2^2]\}.
\end{equation*}
Since SGMs use an affine drift term, the transition kernel $\text{p}(\textbf{x}^s|\textbf{x}^0)$ follows a certain Gaussian distribution~\cite{sarkka2019applied} and therefore, ${\nabla}_{\textbf{x}^s}\log p(\textbf{x}^s|\textbf{x}^0)$ can be analytically calculated.

\subsubsection{Autoregressiveness of Diffusion Model}\label{autosgm}

Let $\textbf{x}_{1:N}$ be a time-series sample which consists of $N$ observations. 
In order to synthesize time-series $\textbf{x}_{1:N}$, unlike other generation tasks, we must consider autoregressiveness: generating each observation $\textbf{x}_n$ at sequential order $n \in \{2,...,N\}$ considering its previous history $\textbf{x}_{1:n-1}$. One can train neural networks to learn the conditional likelihood $\text{p}(\textbf{x}_n|\textbf{x}_{1:n-1})$ and generate each $\textbf{x}_n$ recursively using it, depicting so-called autoregressive property of time-series domain.

Up to our survey, there's no paper about diffusion models considering autoregressiveness in time-series generation. Instead of it, A researcher dealt with a more generalized case, autoregressive conditional score matching to make its score model, $M_{\theta}(s,\cdot,\cdot)$, to be $M_{\theta}(s,\textbf{x}_{n}^s,\textbf{x}_{1:n-1}^0)\sim{\nabla}_{\textbf{x}_{n}^s}\log p(\textbf{x}_{n}^s|\textbf{x}_{1:n-1}^0)$~\cite{DBLP:journals/corr/abs-2010-12810}. Inspired by the previous work, we enlighten a connection between SGMs and time-series generation by deriving an autoregressive denoising score matching (see Theorem~\ref{thm1}). We emphasize that TSGM is fundamentally different in that (i) TSGM uses its own autoregressive denoising score matching loss,  (ii) as a result, it optimizes not ${\nabla}_{\textbf{x}_{n}^s}\log p(\textbf{x}_{n}^s|\textbf{x}_{1:n-1}^0)$ but ${\nabla}_{\textbf{x}_{1:n}^s}\log p(\textbf{x}_{1:n}^s|\textbf{x}_{1:n-1}^0)$, which is fit for its RNN-based (markovian) encoder and decoder.

\vspace{-0.5em}

\subsection{Time-series Generation}\label{sec:time_series_generation}

There are several time-series generation papers, and we introduce their ideas. TimeVAE~\cite{desai2021timevae} is a variational autoencoder to synthesize time-series data. This model can provide interpretable results by reflecting temporal structures such as trend and seasonality in the generation process. 
CTFP~\cite{NEURIPS2020_58c54802} is a well-known normalizing flow model. It can treat both regular and irreugular time-series data by a deformation of the standard Wiener process.

TimeGAN~\cite{yoon2019timegan} uses a GAN architecture to generate time-series. First, it trains an encoder and decoder, which transform a time-series sample $\textbf{x}_{1:N}$ into latent vectors $\textbf{h}_{1:N}$ and recover them by using a recurrent neural network (RNN). Next, it trains a generator and discriminator pair on latent space, by minimizing the discrepancy between an estimated and true distribution, i.e. $\hat{p}(\textbf{x}_n|\textbf{x}_{1:n-1})$ and ${p}(\textbf{x}_n|\textbf{x}_{1:n-1})$. Since it uses an RNN-based encoder, it can efficiently learn the conditional likelihood ${p}(\textbf{x}_n|\textbf{x}_{1:n-1})$ by treating it as ${p}(\textbf{h}_n|\textbf{h}_{n-1})$, since $\textbf{h}_n\sim\textbf{x}_{1:n}$ under the regime of RNNs. Therefore, it can generate each observation $\textbf{x}_n$ considering its previous history $\textbf{x}_{1:n-1}$. However, GAN-based generative models are vulnerable to the issue of mode collapse~\cite{xiao2022tackling} and unstable behavior problems during training~\cite{chu2020unstable}. GT-GAN~\cite{https://doi.org/10.48550/arxiv.2210.02040} attempted to solve the problems by incorporating an invertible neural network-based generator into its framework. There also exist GAN-based methods to generate other types of sequential data, e.g., video, sound, etc~\cite{esteban2017rcgan,mogren2016crnngan,xu2020cotgan,donahue2019wavegan}. In our experiments, we also use them as our baselines for thorough evaluations.

\begin{figure*}[t]
\centering
\includegraphics[width=1\textwidth]{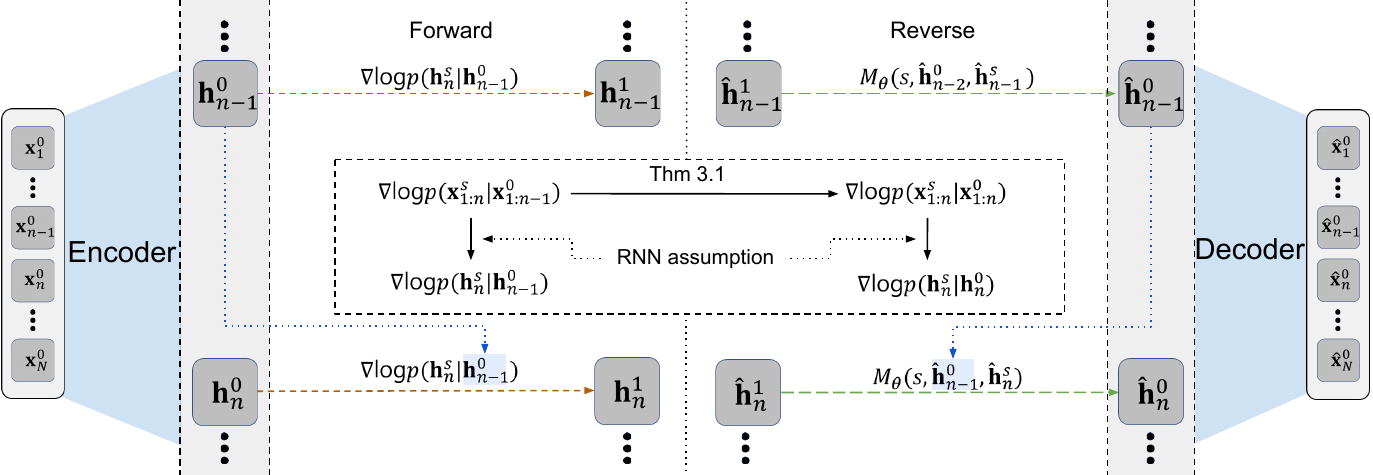}
\caption{The overall workflow of TSGM (see Section~\ref{sec:train}). Our original learning objective is to approximate $\nabla\log p({\textbf{x}}_{1:n}^s|\textbf{x}_{1:n-1}^0)$, which is computationally prohibitive, with the conditional score network $M_{\theta}(s,{\textbf{x}}_{1:n}^s,\textbf{x}_{1:n-1}^0)$ using an MSE loss. We then prove in Thm.~\ref{thm1} that learning $\nabla\log p({\textbf{x}}_{1:n}^s|\textbf{x}_{1:n}^0)$ is equivalent to $\nabla\log p({\textbf{x}}_{1:n}^s|\textbf{x}_{1:n-1}^0)$ for $\theta$ of $M_{\theta}$ in the MSE loss, i.e., their optimal model parameter $\theta$ is identical. At the end, our score network $M_{\theta}(s,\textbf{h}_{n}^s,\textbf{h}_{n-1}^0)$ learns $\nabla\log p({\textbf{h}}_{n}^s|\textbf{h}_{n})$ since RNNs can encode $\textbf{x}_{1:n}^0$ and $\textbf{x}_{1:n-1}^0$ into their hidden states $\textbf{h}_{n}^0$ and $\textbf{h}_{n-1}^0$, respectively.}
\vspace{-1.5em}
\label{fig1}
\end{figure*}

\section{Proposed Method}

Our proposed TSGM consists of three networks: an encoder, a decoder, and a conditional score network (cf. Fig.~\ref{fig1}). Firstly, we train the encoder and the decoder to connect between time-series samples and a latent space. Next, using the pre-trained encoder and decoder, we train the conditional score network on the latent space. The conditional score network will be used for sampling fake time-series on the latent space.

\subsection{Problem Formulation}
Let $\mathcal{X}$ and $\mathcal{H}$ denote a data space and a latent space, respectively. We define $\mathbf{x}_{1:N}$ as a time-series sample with a sequential length of $N$, and $\mathbf{x}_n$ is a multi-dimensional observation of $\mathbf{x}_{1:N}$ at sequential order $n$. Similarly, $\mathbf{h}_{1:N}$ (resp. $\mathbf{h}_n$) denotes an embedded time series (resp. an embedded observation).

Each observation $\mathbf{x}_n$ can be represented as a pair of time and features, i.e., $\mathbf{x}_n=(t_n,\mathbf{u}(t_n))$, where $t_n\in\mathbb{R}_{\geq0}$ is a time stamp of feature $\mathbf{u}(t_n)\in{\mathbb{R}^{\dim(\mathbf{u})}}$, and $\dim(\mathbf{u})$ is a feature dimension.
$\mathcal{X}$ can be classified into two types: regular time-series and irregular time-series. For irregular setting, we randomly remove 30\%, 50\% and 70\% of samples from regular time-series. Therefore, the only difference between these types is whether time intervals, ${\{t_{n+1}-t_n\}}_{n=1}^{N-1}$, are the same or not.

Our irregular setting considers generating complete time-series given training data with missing values. Consequently, we generate total time-series on both regular and irregular settings and compare the generated ones with the original regular data on evaluation process.

\subsection{Encoder and Decoder}\label{sec:enc}
The encoder and decoder have the task of mapping time-series data to a latent space and vice versa.
We define $e$ and $d$ as an encoding function mapping $\mathcal{X}$ to $\mathcal{H}$ and a decoding function mapping $\mathcal{H}$ to $\mathcal{X}$, respectively. For simplicity but without loss of generality, we utilize an autoregressive autoencoder: RNN-based ones for regular time-series~\cite{DBLP:journals/corr/ChoMGBSB14} and Neural CDE, GRU-ODE for irregular setting~\cite{DBLP:journals/corr/abs-2005-08926,DBLP:journals/corr/abs-1905-12374}. In this section, we describe the encoder and decoder of both regular and irregular time-series generation, respectively.

For regular time-series generation, the encoder $e$ and the decoder $d$ consist of recurrent neural networks, e.g., gated recurrent units (GRUs)~\cite{DBLP:journals/corr/ChoMGBSB14}. Since we use RNNs, both $e$ and $d$ are defined recursively as follows:
\begin{equation}\label{eq:auto}
    \mathbf{h}_n = e(\mathbf{h}_{n-1}, \mathbf{x}_n), \qquad \hat{\mathbf{x}}_n = d(\mathbf{h}_n),
\end{equation}where $\hat{\mathbf{x}}_n$ denotes a reconstructed time-series sample at sequential order $n$. It is well-known that RNN was devised to efficiently handle variable sequences by summarizing past observations~\cite{DBLP:journals/corr/ChoMBB14}. For example, there is a paper used RNN-based encoder and decoder to provide a reversible mapping between features and latent representations, thereby reducing the high-dimensionality of the adversarial learning space, which is well supported by $\textbf{h}_n \sim \textbf{x}_{1:n}$~\cite{yoon2019timegan}.

To process irregular time-series, one can use continuous-time methods for constructing the encoder and the decoder. In our case, we use neural controlled differential equations (NCDEs) for designing the encoder and GRU-ODEs for designing the decoder, respectively~\cite{DBLP:journals/corr/abs-2005-08926, DBLP:journals/corr/abs-1905-12374}. Our encoder based on NCDEs can be defined as follows:
\begin{align*}
\mathbf{h}(t_n) = \mathbf{h}(t_{n-1})+\int_{t_{n-1}}^{t_n}f(t,\mathbf{h}(t); \theta_f){dX(t)\over dt}dt,
\end{align*} where $X(t)$ is an interpolated continuous path from $\textbf{x}_{1:N}$ --- NCDEs typically use the natural cubic spline algorithm to define $X(t)$, which is twice differentiable and therefore, there is not any problem to be used for forward inference and backward training. In other words, NCDEs evolve the hidden state $\mathbf{h}(t)$ by solving the above Riemann-Stieltjes integral.

For the decoder, one can use the following GRU-ODE-based definition:
\begin{align*}
&\overline{\mathbf{d}}(t_n)=\mathbf{d}(t_{n-1})+\int_{t_{n-1}}^{t_n}g(t,\mathbf{d}(t); \theta_g)dt, \\ &\mathbf{d}(t_n)=\text{GRU}(\textbf{h}(t_{n}),\overline{\textbf{d}}(t_{n})), \qquad \hat{\mathbf{x}}_n = FC(\mathbf{d}(t_n)),
\end{align*}where $FC$ denotes a fully-connected layer-based output layer. The intermediate hidden representation $\overline{\mathbf{d}}(t_n)$ is jumped into the hidden representation $\mathbf{d}(t_n)$ by the GRU-based jump layer. At the end, there is an output layer. We take the architecture for functions $f$, $g$ described in Table~\ref{table9}.

\begin{table}[h]
\centering
        \begin{tabular}{c|c|c}
            \hline
            Layer & Activation function & Linear \\
            \hline
            1 & ReLU & {{dim}(\textbf{x})$\rightarrow$ 4 dim(\textbf{x})}\\
            2 & ReLU & {4 {dim}(\textbf{x})$\rightarrow$ 4 {dim}(\textbf{x})}\\
            3 & ReLU & {4 {dim}(\textbf{x})$\rightarrow$ 4 {dim}(\textbf{x})}\\
            4 & Tanh & {4 {dim}(\textbf{x})$\rightarrow$ {dim}(\textbf{x})}\\
            \hline
        \end{tabular}
        \\
        \vspace{2em}
        \begin{tabular}{c|c|c|c}
            \hline
            Layer & Gate & Activation function & Linear \\
            \hline
            \multirow{3}{*}{1} & $r_t$ & ReLU & \multirow{3}{*}{{{dim}(\textbf{h})$\rightarrow$ dim(\textbf{h})}}\\
            & $z_t$ & ReLU & \\
            & $u_t$ & Tanh & \\
            \hline
        \end{tabular}

\caption{Architecture of functions $f$(upper) and $g$(lower). Each layer of encoder and gate of decoder takes $(\sigma \circ \text{Linear})$ form where $\sigma$ denotes activation function. We describe which activation and Linear function are used.}\label{table9}
\end{table}

For our irregular time-series experiments, i.e, dropping 30\%, 50\%, and 70\% of observations from regular time-series, we use the above encoder and decoder definitions and have good results.

After embedding real time-series data onto a latent space, we can train the conditional score network with its conditional log-likelihood, whose architecture is described in the following Section~\ref{arch}. The encoder and decoder are pre-trained before our main training.

\subsection{Conditional Score Network}\label{arch}

Unlike other generation tasks, e.g., image generation~\cite{song2021SDE} and tabular data synthesis~\cite{kim2022sos}, where each sample is independent, time-series observations are dependent to their past observations. Therefore, the score network for time-series generation must be designed to learn the conditional log-likelihood given past generated observations, which is more complicated than that in image generation. 

In order to learn the conditional log-likelihood, we modify the popular U-net~\cite{ronneberger2015unet} architecture for our purposes. Since U-net has achieved various excellent results for other generative tasks~\cite{song2019smld,song2021SDE}, we modify its 2-dimensional convolution layers to 1-dimensional ones for handling time-series observations. The modified U-net, denoted $M_{\theta}$, is trained to learn our conditional score function (cf. Eq.~\eqref{eq:ourscore}). More details on training and sampling with $M_{\theta}$ are in Sec.~\ref{sec:sampling}.

\subsection{Training Objective Function}\label{sec:train}
\paragraph{Loss for autoencoder} We use two training objective functions. First, we train the encoder and the decoder using $L_{ed}$. Let $\mathbf{x}_{1:N}^0 \sim p(\mathbf{x}_{1:N}^0)$ and $\hat{\mathbf{x}}_{1:N}^0$ denote an real time-series sample and its reconstructed copy by the encoder-decoder process, respectively. Then, $L_{ed}$ denotes the following MSE loss between $\mathbf{x}_{1:N}^0$ and its reconstructed copy $\hat{\mathbf{x}}_{1:N}^0$:
\begin{align*}
L_{ed} = \mathbb{E}_{\mathbf{x}_{1:N}^0}[\left\|\hat{\mathbf{x}}_{1:N}^0 - \mathbf{x}_{1:N}^0 \right\|_2^2].
\end{align*}

\paragraph{Loss for score network} Next, we define another loss $L_{score}^{\mathcal{H}}$ in Eq.~\eqref{eq:ourscore} to train the conditional score network $M_{\theta}$, which is one of our main contributions. In order to derive the training loss $L_{score}^{\mathcal{H}}$ from the initial loss definition $L_{1}$, we describe its step-by-step derivation procedure. At sequential order $n$ in $\{1,...,N\}$, we diffuse $\mathbf{x}_{1:n}^0$ through a sufficiently large number of steps of the forward SDE to a Gaussian distribution. Let $\mathbf{x}_{1:n}^s$ denotes a diffused sample at step $s \in [0,1]$ from $\mathbf{x}_{1:n}^0$. Then, the conditional score network $M_{\theta}({s, \mathbf{x}}_{1:n}^s, \mathbf{x}_{1:n-1}^0)$ can be trained to learn the gradient of the conditional log-likelihood with the following $L_{1}$ loss: 
\begin{equation*}
L_{1} = \mathbb{E}_{s}\mathbb{E}_{{\textbf{x}}_{1:N}^0}\left[\sum_{n=1}^{N}\lambda(s)l_{1}(n,s)\right], 
\end{equation*}
where 
\begin{footnotesize}
\begin{equation*}
l_{1}(n,s) = \mathbb{E}_{{\textbf{x}}_{1:n}^s}\left[\left\|M_{\theta}(s,{\textbf{x}}_{1:n}^s,\textbf{x}_{1:n-1}^0) - {\nabla}_{{\textbf{x}}_{1:n}^s}\log p({\textbf{x}}_{1:n}^s|\textbf{x}_{1:n-1}^0)\right\|_2^2\right].
\end{equation*}
\end{footnotesize}

In the above definition, ${\nabla}_{{\textbf{x}}_{1:n}^s}\log p({\textbf{x}}_{1:n}^s|\textbf{x}_{1:n-1}^0)$, where $\textbf{x}_{i}^0$ depends on $\textbf{x}_{1:i-1}^0$ for each $i \in \{2,...,n\}$, is designed specially for time-series generation. Note that for our training, ${\textbf{x}}_{1:n}^s$ is sampled from ${p}({\textbf{\textbf{x}}}_{1:n}^s|\textbf{x}_{1:n-1}^0)$, and $s$ is uniformly sampled from $[0,1]$.

However, using the above formula, which is a na\"ive score matching on time-series, is
computationally prohibitive~\cite{JMLR:v6:hyvarinen05a,DBLP:journals/corr/abs-1905-07088}. Thanks to the following theorem, the more efficient denoising score loss $L_{score}$ can be defined. 

\begin{theorem}[Autoregressive denoising score matching]\label{thm1}
$l_{1}(n,s)$ can be replaced with the following $l_{2}(n,s)$
\begin{equation*}
L_{score} = \mathbb{E}_{s}\mathbb{E}_{\textbf{\textsc{x}}_{1:N}^0}\left[\sum_{n=1}^{N}\lambda(s)l_2(n, s) \right],
\end{equation*}where
\begin{equation*}
l_{2}(n,s) = \mathbb{E}_{{\textbf{\textsc{x}}}_{1:n}^s}\left[\left\|M_{\theta}(s,{\textbf{\textsc{x}}}_{1:n}^s,\textbf{\textsc{x}}_{1:n-1}^0) - 
{\nabla}_{{\textbf{\textsc{x}}}_{1:n}^s}\log p({\textbf{\textsc{x}}}_{1:n}^s|\textbf{\textsc{x}}_{1:n}^0)\right\|_2^2\right].
\end{equation*}
Then, $L_{1}=L_{score}$ is satisfied.\qed
\end{theorem}

\noindent\emph{Sketch of proof.}
(A complete, line-by-line proof will be included in the arXiv version of this paper.) Define 
$f(\mathbf{x}_{1:n}^0):=\mathbb{E}_{s}\mathbb{E}_{\mathbf{x}_{1:n}^s}\!\left[\lambda(s)\,\|M_\theta-\nabla_{\mathbf{x}_{1:n}^s}\log p(\mathbf{x}_{1:n}^s\!\mid\!\mathbf{x}_{1:n}^0)\|_2^2\right]$.
Then
$L_2=\sum_{n=1}^N \mathbb{E}_{\mathbf{x}_{1:N}^0}\mathbb{E}_{\mathbf{x}_n^0\mid \mathbf{x}_{1:n-1}^0}[f(\mathbf{x}_{1:n}^0)]$.
By the law of total expectation (marginalizing the suffix $\mathbf{x}_{n+1:N}^0$) and linearity of expectation, this equals 
$\mathbb{E}_{\mathbf{x}_{1:N}^0}\!\left[\sum_{n=1}^N f(\mathbf{x}_{1:n}^0)\right]=L_{\text{score}}$,
so $L_1=L_{\text{score}}$.

Since we pre-train the encoder and decoder, the encoder can embed $\textbf{x}_{1:n}^0$ into $\textbf{h}_n^0 \in \mathcal{H}$. Ideally, $\textbf{h}_n^0$ involves the entire information of $\textbf{x}_{1:n}^0$.
Therefore, $L_{score}$ can be re-written as follows with the embeddings in the latent space:
\begin{equation}\label{eq:ourscore}
 L_{score}^{\mathcal{H}} = \mathbb{E}_{s}\mathbb{E}_{\textbf{h}_{1:N}^0}\sum_{n=1}^{N}\left[\lambda(s)l_3(n, s)\right],
\end{equation}
\noindent with {\small $l_{3}(n, s) = \mathbb{E}_{\textbf{h}_{n}^s}\left[\left\|M_{\theta}(s,\textbf{h}_{n}^s,\textbf{h}_{n-1}^0)-{\nabla}_{\textbf{h}_{n}^s}\log p(\textbf{h}_{n}^s|\textbf{h}_{n}^0)\right\|_2^2\right]$}. $L_{score}^{\mathcal{H}}$ is what we use for our experiments (instead of $L_{score}$). Until now, we introduced our target objective functions, $L_{ed}$ and $L_{score}^{\mathcal{H}}$. We note that we use exactly the same weight $\lambda(s)$ as that in~\cite{song2021SDE}.

\subsection{Training and Sampling Procedures}\label{sec:sampling}

\paragraph{Training method} We explain details of our training method.
At first, we pre-train both the encoder and decoder using $L_{ed}$. After pre-training them, we train the conditional score network. When training the latter one, we use the embedded hidden vectors produced by the encoder. After encoding an input $\mathbf{x}_{1:N}^0$, we obtain its latent vectors $\mathbf{h}_{1:N}^0$ --- we note that each hidden vector $\mathbf{h}_n^0$ has all the previous information from 1 to $n$ for the RNN-based encoder's autoregressive property as shown in the Equation~\eqref{eq:auto}. We use the following forward process~\cite{song2021SDE}, where $n$ means the sequence order of the input time-series, and $s$ denotes the time (or step) of the diffusion step :
\begin{equation*}
    d\textbf{h}_n^s=\textbf{f}(s,\textbf{h}_n^s)ds+g(s)d\textbf{w}, \qquad s \in [0,1].
\end{equation*}

Note that we only use the VP and subVP-based TSGM in our experiments and exclude the VE-based one for its inferiority for time series synthesis in our experiments. During the forward process, the conditional score network reads the pair ($s$, $\mathbf{h}_n^s$, $\mathbf{h}_{n-1}^0$) as input and thereby, it can learn the conditional score function $\nabla \log p(\mathbf{h}_n^s | \mathbf{h}_{n-1}^0)$ by using $L_{score}^{\mathcal{H}}$, where $\mathbf{h}_{0}^0 = \mathbf{0}$.

Although we basically train the conditional score network and the encoder-decoder pair alternately after the pre-training step, for some datasets, we found that training only the conditional score network achieves better results after pre-training the autoencoder. Therefore, $use_{alt}=\{True, False\}$ is a hyperparameter to set whether we use the alternating training method.

\paragraph{Sampling method} After the training procedure, we use the following conditional reverse process:
\begin{equation*}
    d\textbf{h}_n^s=[\textbf{f}(s,\textbf{h}_n^s)-g^2(s)\nabla_{\textbf{h}_n^s}{\log p(\textbf{h}_n^s|\textbf{h}_{n-1}^0)}]ds+g(s)d\bar{\textbf{w}},
\end{equation*}where $s$ is uniformly sampled over $[0,1]$. The conditional score function in this process can be replaced with the trained score network $M_{\theta}(s,\textbf{h}_{n}^s,\textbf{h}_{n-1}^0)$. The detailed sampling method is as follows:
\begin{enumerate}
    \item At first, we sample $\textbf{z}_1$ from a Gaussian prior distribution and set $\textbf{h}_{1}^1 = \textbf{z}_1$ and $\mathbf{h}_{0}^0  = \mathbf{0}$. 
    We then generates an initial observation $\hat{\textbf{h}}_1^0$ by denoising $\textbf{h}_{1}^1$ following the conditional reverse process with $M_{\theta}(s, \textbf{h}_{n}^s, \mathbf{h}_{0}^0)$ via the \textit{predictor-corrector} method~\cite{song2021SDE}.

    \item We repeat the following computation for every $2 \leq n \leq N$, i.e., recursive generation. We sample $\textbf{z}_{n}$ from a Gaussian prior distribution and set $\textbf{h}_{n}^1 = \textbf{z}_{n}$ for $n \in \{2,...,N\}$. After reading the previously generated samples $\hat{\textbf{h}}_{n-1}^0$, we then denoise $\textbf{h}_{n}^1$ following the conditional reverse process with $M_{\theta}(s,\textbf{h}_{n}^s,\textbf{h}_{n-1}^0)$ to generate $\hat{\textbf{h}}_{n}^0$ via the \textit{predictor-corrector} method.
\end{enumerate}
  
\noindent Once the sampling procedure is finished, we can reconstruct ${\hat{\textbf{x}}}_{1:N}^0$ from $\hat{\textbf{h}}_{1:N}^0$ using the trained decoder at once.

\begin{table*}[t]
\centering
\vspace{-1em}
\footnotesize
\renewcommand{\arraystretch}{1.0}
\setlength{\tabcolsep}{7pt}
\vspace{1.0em}
\resizebox{\textwidth}{!}
{
        \begin{tabular}{c|c|cccc|cccc}
            \toprule
             & \multirow{2}{*}{Method} & \multicolumn{4}{c|}{Regular Settings} & \multicolumn{4}{c}{Irregular Settings (Missing Rate: 30\%)} \\
             & & Stocks & Energy & Air & AI4I & Stocks & Energy & Air & AI4I\\ \midrule
            \parbox[t]{2mm}{\multirow{14}{*}{\rotatebox{90}{Disc. score}}} & TSGM-VP & \underline{.022}$\pm$\underline{.005} & \underline{.221}$\pm$\underline{.025} & \underline{.122}$\pm$\underline{.014} & .147$\pm$.005 & .062$\pm$.018 & \textbf{.294}$\pm$\textbf{.007} & \textbf{.190}$\pm$\textbf{.042} & \underline{.142}$\pm$\underline{.048}  \\ 
            & TSGM-subVP & \textbf{.021}$\pm$\textbf{.008} & \textbf{.198}$\pm$\textbf{.025} & .127$\pm$.010 & .150$\pm$.010 & \textbf{.025}$\pm$\textbf{.009} & .326$\pm$.008 & \underline{.240}$\pm$\underline{.018} & \textbf{.121}$\pm$\textbf{.082}\\ \cline{2-10}
            & T-Forcing & .226$\pm$.035 & .483$\pm$.004 & .404$\pm$.020 & .435$\pm$.025 & .409$\pm$.051 & .347$\pm$.046 & .458$\pm$.122 & .493$\pm$.018  \\ 
            & P-Forcing & .257$\pm$.026 & .412$\pm$.006 & .484$\pm$.007 & .443$\pm$.026 & .480$\pm$.060 & .491$\pm$.020 & .494$\pm$.012 & .430$\pm$.061  \\ 
            & TimeGAN & .102$\pm$.031 & .236$\pm$.012 & .447$\pm$.017 & .070$\pm$.009 & .411$\pm$.040 & .479$\pm$.010 & .500$\pm$.001 & .500$\pm$.000  \\
            & RCGAN & .196$\pm$.027 & .336$\pm$.017 & .459$\pm$.104 & .234$\pm$.015 & .500$\pm$.000 & .500$\pm$.000 & .500$\pm$.000 & .500$\pm$.000  \\ 
            & C-RNN-GAN & .399$\pm$.028 & .499$\pm$.001 & .499$\pm$.000 & .499$\pm$.001 & .500$\pm$.000 & .500$\pm$.000 & .500$\pm$.000 & .450$\pm$.150  \\ 
            & TimeVAE & .175$\pm$.031 & .498$\pm$.006 & .381$\pm$.037 & .446$\pm$.024 & .423$\pm$.088 & .382$\pm$.124 & .373$\pm$.191 & .384$\pm$.086  \\
            & COT-GAN & .285$\pm$.030 & .498$\pm$.000 & .423$\pm$.001 & .411$\pm$.018 & .499$\pm$.001 & .500$\pm$.000 & .500$\pm$.000 & .500$\pm$.000  \\ 
            & CTFP & .499$\pm$.000 & .500$\pm$.000 & .499$\pm$.000 & .499$\pm$.001 & .500$\pm$.000 & .500$\pm$.000 & .500$\pm$.000 & .499$\pm$.001  \\ 
            & GT-GAN & .077$\pm$.031 & .221$\pm$.068 & .413$\pm$.001 & .394$\pm$.090 & .251$\pm$.097 & .333$\pm$.063 & .454$\pm$.029 & .435$\pm$.018  \\ 
            & KoVAE & .057$\pm$.038 & .289$\pm$.017 & .268$\pm$.014 & \underline{.069}$\pm$\underline{.031} & \underline{.038}$\pm$\underline{.027} & \underline{.322}$\pm$\underline{.012} & .289$\pm$.031 & .449$\pm$.037  \\ 
            & Diffusion-TS & .089$\pm$.024 & .319$\pm$.021 & \textbf{.118}$\pm$\textbf{.009} & \textbf{.021}$\pm$\textbf{.011} & .499$\pm$.001 & .467$\pm$.088 & .479$\pm$.021 & .437$\pm$.022  \\ 
            \midrule
            \cline{1-10}
            \parbox[t]{2mm}{\multirow{15}{*}{\rotatebox{90}{Pred. score}}} & TSGM-VP & \textbf{.037}$\pm$\textbf{.000} & .257$\pm$.000 & \textbf{.005}$\pm$\textbf{.000} & \textbf{.217}$\pm$\textbf{.000} & \textbf{.012}$\pm$\textbf{.002} & \textbf{.049}$\pm$\textbf{.001} & \textbf{.042}$\pm$\textbf{.002} & \underline{.067}$\pm$\underline{.013}  \\ 
            & TSGM-subVP & \textbf{.037}$\pm$\textbf{.000} & \underline{.252}$\pm$\underline{.000} & \textbf{.005}$\pm$\textbf{.000} & \textbf{.217}$\pm$\textbf{.000} & \textbf{.012}$\pm$\textbf{.001} & \textbf{.049}$\pm$\textbf{.001} & \underline{.044}$\pm$\underline{.004} & \textbf{.061}$\pm$\textbf{.001}\\ \cline{2-10}
            & T-Forcing & \underline{.038}$\pm$\underline{.001} & .315$\pm$.005 & .008$\pm$.000 & .242$\pm$.001 & .027$\pm$.002 & \underline{.090}$\pm$\underline{.001} & .112$\pm$.004 & .147$\pm$.010  \\ 
            & P-Forcing & .043$\pm$.001 & .303$\pm$.006 & .021$\pm$.000 & \underline{.220}$\pm$\underline{.000} & .079$\pm$.008 & .147$\pm$.001 & .101$\pm$.003 & .134$\pm$.005  \\ 
            & TimeGAN & .038$\pm$.001 & .273$\pm$.004 & .017$\pm$.004 & .253$\pm$.002 & .105$\pm$.053 & .248$\pm$.024 & .325$\pm$.005 & .251$\pm$.010  \\ 
            & RCGAN & .040$\pm$.001 & .292$\pm$.005 & .043$\pm$.000 & .224$\pm$.001 & .523$\pm$.020 & .409$\pm$.020 & .342$\pm$.018 & .329$\pm$.037  \\ 
            & C-RNN-GAN & .038$\pm$.000 & .483$\pm$.005 & .111$\pm$.000 & .340$\pm$.006 & .345$\pm$.002 & .440$\pm$.000 & .354$\pm$.060 & .400$\pm$.026  \\ 
            & TimeVAE & .042$\pm$.002 & .268$\pm$.004 & .013$\pm$.002 & .233$\pm$.010 & .207$\pm$.014 & .139$\pm$.004 & .105$\pm$.002 & .144$\pm$.003  \\ 
            & COT-GAN & .044$\pm$.000 & .260$\pm$.000 & .024$\pm$.001 & \underline{.220}$\pm$\underline{.000} & .274$\pm$.000 & .427$\pm$.000 & .451$\pm$.000 & .570$\pm$.000  \\ 
            & CTFP & .084$\pm$.005 & .469$\pm$.008 & .476$\pm$.235 & .412$\pm$.024 & .070$\pm$.009 & .499$\pm$.000 & .060$\pm$.027 & .424$\pm$.002  \\ 
            & GT-GAN & .040$\pm$.000 & .312$\pm$.002 & .007$\pm$.000 & .239$\pm$.000 & .077$\pm$.031 & .221$\pm$.068 & .064$\pm$.002 & .087$\pm$.013  \\
            & KoVAE & \textbf{.037}$\pm$\textbf{.000} & \textbf{.251}$\pm$\textbf{.000} & .039$\pm$.002 & .221$\pm$.001 & \underline{.014}$\pm$\underline{.003} & \textbf{.049}$\pm$\textbf{.001} & .057$\pm$.007 & \underline{.067}$\pm$\underline{.011}  \\
            & Diffusion-TS & \textbf{.037}$\pm$\textbf{.000} & \textbf{.251}$\pm$\textbf{.000} & \underline{.006}$\pm$\underline{.000} & \textbf{.217}$\pm$\textbf{.000} & .257$\pm$.001 & .164$\pm$.004 & .100$\pm$.001 & .141$\pm$.013  \\
            \cline{2-10}
            & Original & .036$\pm$.001 & .250$\pm$.003 & .004$\pm$.000 & .217$\pm$.000 & .011$\pm$.002 & .045$\pm$.001 & .044$\pm$.006 & .059$\pm$.001 \\ 
            \bottomrule
            \multicolumn{10}{c}{} \\ 
        \end{tabular}}
\vspace{-1.5em}
\caption{The left and right ones denote experimental results on regular time-series and irregular time-series with 30\% missing rates, respectively. Results for higher missing rates are in Table~\ref{table10}. 
Note that except for our method (TSGM), CTFP, and GT-GAN, the other methods cannot deal with irregular time series, so we make it possible for them to operate on irregular settings by replacing RNN encoder with GRU-D.}
\label{table2_1}
\vspace{-1.em}
\end{table*}

\begin{table*}[t]
\centering
\footnotesize
\renewcommand{\arraystretch}{1.0}
\setlength{\tabcolsep}{7pt}
\vspace{1.0em}
\resizebox{\textwidth}{!}
{
        \begin{tabular}{c|c|cccc|cccc}
            \toprule
            & \multirow{2}{*}{Method} & \multicolumn{4}{c|}{Irregular Settings (Missing Rate: 50\%)} & \multicolumn{4}{c}{Irregular Settings (Missing Rate: 70\%)} \\
             & & Stocks & Energy & Air & AI4I & Stocks & Energy & Air & AI4I\\ \midrule
            \parbox[t]{2mm}{\multirow{14}{*}{\rotatebox{90}{Disc. score}}} & TSGM-VP & \underline{.051}$\pm$\underline{.014} & .398$\pm$.003 & .272$\pm$.012 & \underline{.156}$\pm$\underline{.106} & \underline{.065}$\pm$\underline{.010} & .482$\pm$.003 & .337$\pm$.025 & \underline{.327}$\pm$\underline{.104}  \\ 
            & TSGM-subVP & \textbf{.031}$\pm$\textbf{.012} & .421$\pm$.008 & \textbf{.213}$\pm$\textbf{.025} & \textbf{.137}$\pm$\textbf{.102} & \textbf{.035}$\pm$\textbf{.009} & \textbf{.213}$\pm$\textbf{.025} & .329$\pm$.027 & \textbf{.235}$\pm$\textbf{.123}\\ \cline{2-10}
            & T-Forcing & .407$\pm$.034 & \underline{.376}$\pm$\underline{.046} & .499$\pm$.001 & .473$\pm$.045 & .404$\pm$.068 & .336$\pm$.032 & .499$\pm$.001 & .493$\pm$.010  \\ 
         & P-Forcing & .500$\pm$.000 & .500$\pm$.000 & .494$\pm$.012 & .437$\pm$.079 & .449$\pm$.150 & .494$\pm$.011 & .498$\pm$.002 & .440$\pm$.125  \\ 
         & TimeGAN & .477$\pm$.021 & .473$\pm$.015 & .500$\pm$.001 & .500$\pm$.000 & .485$\pm$.022 & .500$\pm$.000 & .500$\pm$.000 & .500$\pm$.000  \\ 
         & RCGAN & .500$\pm$.000 & .500$\pm$.000 & .500$\pm$.000 & .500$\pm$.000 & .500$\pm$.000 & .500$\pm$.000 & .500$\pm$.000 & .500$\pm$.000  \\ 
         & C-RNN-GAN & .500$\pm$.000 & .500$\pm$.000 & .500$\pm$.000 & .450$\pm$.150 & .500$\pm$.000 & .500$\pm$.000 & .500$\pm$.000 & .500$\pm$.000  \\ 
         & TimeVAE & .411$\pm$.110 & .436$\pm$.088 & .423$\pm$.153 & .389$\pm$.113 & .444$\pm$.148 & .498$\pm$.003 & .426$\pm$.148 & .371$\pm$.092  \\ 
         & COT-GAN & .499$\pm$.001 & .500$\pm$.000 & .500$\pm$.000 & .500$\pm$.000 & .498$\pm$.001 & .500$\pm$.000 & .500$\pm$.000 & .500$\pm$.000  \\ 
         & CTFP & .499$\pm$.000 & .500$\pm$.000 & .500$\pm$.000 & .499$\pm$.001 & .500$\pm$.000 & .500$\pm$.000 & .500$\pm$.000 & .499$\pm$.000  \\ 
         & GT-GAN & .265$\pm$.073 & \textbf{.317}$\pm$\textbf{.010} & .434$\pm$.035 & .276$\pm$.033 & .230$\pm$.053 & \underline{.325}$\pm$\underline{.047} & .444$\pm$.019 & .362$\pm$.043  \\ 
         & KoVAE & .104$\pm$.088 & \textbf{.317}$\pm$\textbf{.023} & \underline{.268}$\pm$\underline{.023} & .453$\pm$.040 & .085$\pm$.073 & .336$\pm$.040 & \underline{.323}$\pm$\underline{.029} & .463$\pm$.034  \\ 
         & Diffusion-TS & .465$\pm$.036 & .477$\pm$.028 & .435$\pm$.051 & .332$\pm$.071 & .463$\pm$.055 & .455$\pm$.068 & \textbf{.302}$\pm$\textbf{.154} & .404$\pm$.069  \\ 
         \cline{1-10}
            \midrule
            \parbox[t]{2mm}{\multirow{15}{*}{\rotatebox{90}{Pred. score}}} & TSGM-VP & \textbf{.011}$\pm$\textbf{.000} & \underline{.051}$\pm$\underline{.001} & \textbf{.041}$\pm$\textbf{.001} & \textbf{.060}$\pm$\textbf{.001} & \textbf{.011}$\pm$\textbf{.000} & .053$\pm$.001 & \underline{.043}$\pm$\underline{.000} & \underline{.092}$\pm$\underline{.024} \\ 
            & TSGM-subVP & \textbf{.011}$\pm$\textbf{.000} & \underline{.051}$\pm$\underline{.001} & \underline{.042}$\pm$\underline{.002} & \underline{.065}$\pm$\underline{.013} & \underline{.012}$\pm$\underline{.000} & \textbf{.042}$\pm$\textbf{.002} & \textbf{.042}$\pm$\textbf{.001} & .097$\pm$.020 \\ \cline{2-10}
            & T-Forcing & .038$\pm$.003 & .090$\pm$.000 & .121$\pm$.003 & .143$\pm$.005 & .031$\pm$.002 & .091$\pm$.000 & .116$\pm$.003 & .144$\pm$.004  \\ 
         & P-Forcing & .089$\pm$.010 & .198$\pm$.005 & .101$\pm$.003 & .116$\pm$.007 & .107$\pm$.009 & .193$\pm$.006 & .107$\pm$.002 & .125$\pm$.007  \\ 
         & TimeGAN & .254$\pm$.047 & .339$\pm$.029 & .325$\pm$.005 & .251$\pm$.010 & .228$\pm$.000 & .443$\pm$.000 & .425$\pm$.008 & .323$\pm$.011  \\ 
         & RCGAN & .333$\pm$.044 & .250$\pm$.010 & .335$\pm$.023 & .276$\pm$.066 & .441$\pm$.045 & .349$\pm$.027 & .359$\pm$.008 & .346$\pm$.029  \\ 
         & C-RNN-GAN & .273$\pm$.000 & .438$\pm$.000 & .289$\pm$.033 & .373$\pm$.037 & .281$\pm$.019 & .436$\pm$.000 & .306$\pm$.040 & .262$\pm$.053  \\ 
         & TimeVAE & .195$\pm$.012 & .143$\pm$.007 & .103$\pm$.002 & .144$\pm$.004 & .199$\pm$.009 & .134$\pm$.004 & .108$\pm$.004 & .142$\pm$.008  \\ 
         & COT-GAN & .246$\pm$.000 & .475$\pm$.000 & .557$\pm$.000 & .449$\pm$.000 & .278$\pm$.000 & .456$\pm$.000 & .556$\pm$.000 & .435$\pm$.000  \\ 
         & CTFP & .084$\pm$.005 & .469$\pm$.008 & .476$\pm$.235 & .412$\pm$.024 & .084$\pm$.005 & .469$\pm$.008 & .476$\pm$.235 & .412$\pm$.024  \\ 
         & GT-GAN & .018$\pm$.002 & .064$\pm$.001 & .061$\pm$.003 & .113$\pm$.024 & .020$\pm$.005 & .076$\pm$.001 & .059$\pm$.004 & .124$\pm$.003  \\ 
         & KoVAE & \underline{.012}$\pm$\underline{.001} & \textbf{.047}$\pm$\textbf{.001} & .058$\pm$.006 & .075$\pm$.019 & .013$\pm$.001 & \underline{.048}$\pm$\underline{.001} & .062$\pm$.004 & \textbf{.075}$\pm$\textbf{.020}  \\
         & Diffusion-TS & .176$\pm$.007 & .145$\pm$.001 & .089$\pm$.002 & .146$\pm$.004 & .160$\pm$.031 & .145$\pm$.001 & .094$\pm$.000 & .147$\pm$.005  \\
            \cline{2-10}
            & Original & .011$\pm$.002 & .045$\pm$.001 & .044$\pm$.006 & .059$\pm$.001 & .011$\pm$.002 & .045$\pm$.001 & .044$\pm$.006 & .059$\pm$.001 \\ 
            \bottomrule
        \end{tabular}}
\caption{
The left and right ones denote experimental results on irregular time-series with 50\% and 70\% missing rates, respectively.
Note that except for our method (TSGM), CTFP, and GT-GAN, the other methods cannot deal with irregular time series, so we make it possible for them to operate on irregular settings by replacing RNN encoder with GRU-D.}
\label{table10}
\vspace{-1.em}
\end{table*}

\begin{figure*}[t]
\centering
\includegraphics[width=0.7\textwidth]{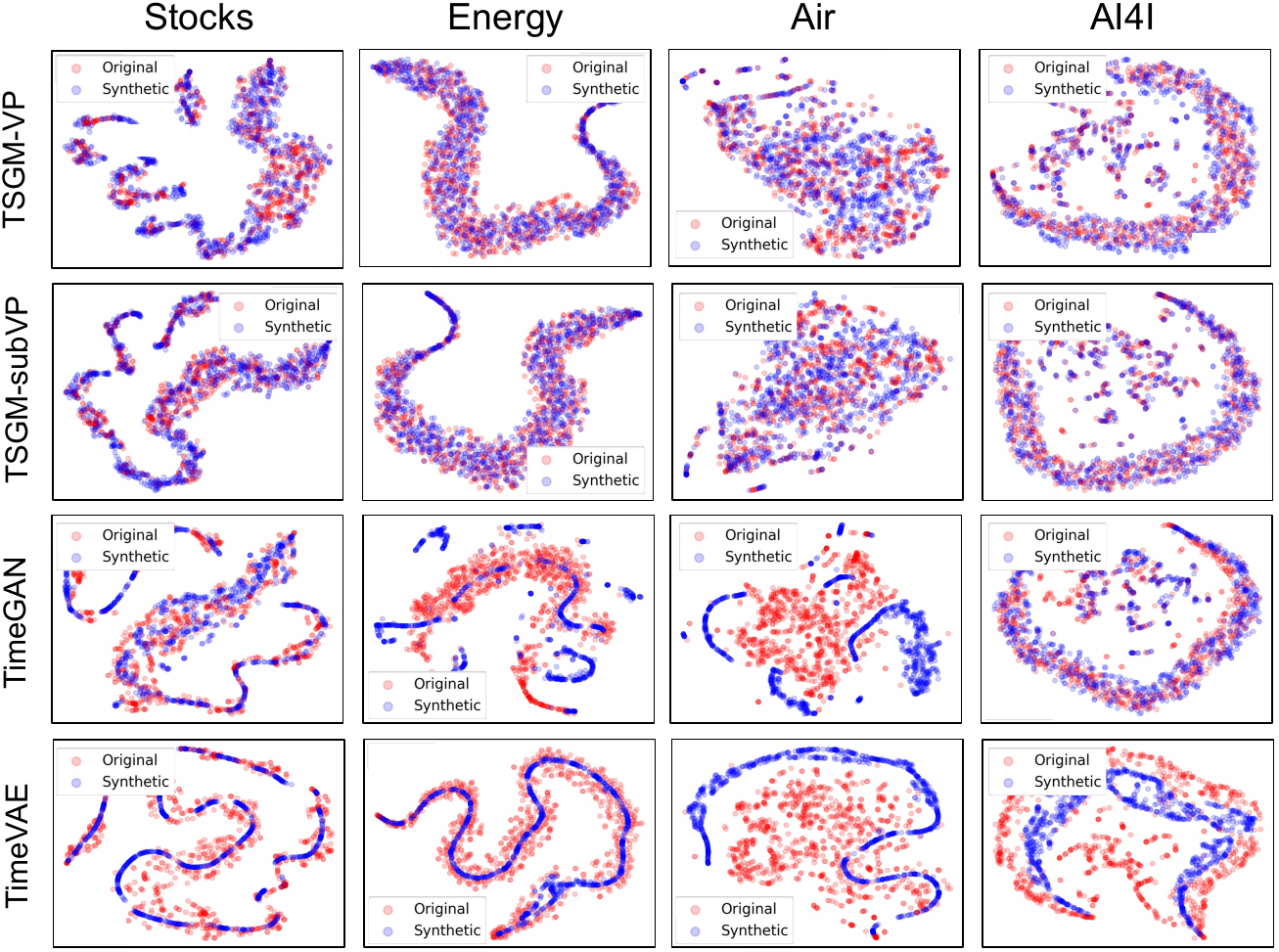} 
\caption{t-SNE plots for TSGM (1st and 2nd columns), TimeGAN (3rd columns), TimeVAE (4th columns), GT-GAN (5th columns) in regular time-series generation. Red and blue dots mean original and synthesized samples, respectively.}
\label{fig2}
\vskip -0.1in
\end{figure*}

\section{Experiments}

\subsection{Experimental Environments}

\subsubsection{Baselines and Datasets}

In the case of the regular time-series generation, we use 4 real-world datasets from various fields with 11 baselines. For the irregular time-series generation, we randomly remove some observations from each time-series sample with 30\%, 50\%, and 70\% missing rates. Therefore, we totally treat 16 datasets, i.e., 4 datasets with one regular and three irregular settings, and 11 baselines.

\begin{itemize}
    \item \textit{Stock}~\cite{yoon2019timegan}: The Google stock dataset was collected irregularly from 2004 to 2019. Each observation has (volume, high, low, opening, closing, adjusted closing prices), and these features are correlated.
   \item \textit{Energy}~\cite{candanedo2017energy}:
   This dataset is from the UCI machine learning repository for predicting the energy use of appliances from highly correlated variables such as house temperature and humidity conditions.
    \item \textit{Air}~\cite{devito2008air}: The UCI Air Quality dataset was collected from 2004 to 2005. Hourly averaged air quality records are gathered using gas sensor devices in an Italian city.
    \item \textit{AI4I}~\cite{matzka2020}: AI4I means the UCI AI4I 2020 Predictive Maintenance dataset. This data reflects the industrial predictive maintenance scenario with correlated features including several physical quantities.
\end{itemize}

We use several types of generative methods for time-series as baselines. At first, we consider autoregressive generative methods: T-Forcing (teacher forcing)~\cite{graves2013tforcing,ilya2011tforcing} and P-Forcing (professor forcing)~\cite{lamb2016pforcing}. 
Next, we use GAN-based methods: TimeGAN~\cite{yoon2019timegan}, RCGAN~\cite{esteban2017rcgan}, C-RNN-GAN~\cite{mogren2016crnngan}, COT-GAN~\cite{xu2020cotgan}, GT-GAN~\cite{https://doi.org/10.48550/arxiv.2210.02040}. We also test VAE-based methods into our baselines: TimeVAE~\cite{desai2021timevae}. Finally, we treat flow-based methods. Among the array of flow-based models designed for time series generation, we have chosen to compare our TSGM against CTFP~\cite{NEURIPS2020_58c54802}. This choice is informed by the fact that CTFP possesses the capability to handle both regular and irregular time series samples, aligning well with the nature of our task which involves generating both regular and irregular time series data. Also, we add some latest baselines, which are famous for using VAE to target irregular time-series generation like GT-GAN and diffusion-based time-series generation model, respectively~\cite{kovae,diffusionts}.

\subsubsection{Evaluation Metrics}

In the image generation domain, researchers have evaluated the \textit{fidelity} and the \textit{diversity} of models by using the Fr\'echet inception distance (FID) and inception score (IS). On the other hand, to measure the fidelity and the diversity of synthesized time-series samples, we use the following predictive score and the discriminative score as in~\cite{yoon2019timegan,https://doi.org/10.48550/arxiv.2210.02040}. We strictly follow the evaluation protocol agreed by the time-series research community~\cite{yoon2019timegan,https://doi.org/10.48550/arxiv.2210.02040}.  Both metrics are designed in a way that lower values are preferred. We run each generative method 10 times with different seeds, and report its mean and standard deviation of the following discriminative and predictive scores:

i) \textit{Predictive Score}: 
We use the predictive score to evaluate whether a generative model can successfully reproduce the temporal properties of the original data. To do this, we first train a popular LSTM-based sequence model for time-series forecasting with synthesized samples. However, the existing predictive score only predicts a last sample of total sequence, so a researcher suggests a comprehensive approach that considers the entire time series~\cite{https://doi.org/10.48550/arxiv.2210.02040}. For fair comparison, we reuse evaluation metrics of previous works for regular and irregular settings~\cite{yoon2019timegan,https://doi.org/10.48550/arxiv.2210.02040}.

ii) \textit{Discriminative Score}:
In order to assess how similar the original and generated samples are, we train a 2-layer LSTM model that classifies the real/fake samples into two classes, real or fake. We use the performance of the trained classifier on the test data as the discriminative score. Therefore, lower discriminator scores mean real and fake samples are similar.

\subsection{Hyperparameters and Miscellaneous Environments}\label{hyper}

$M_{\theta}$ has various hyperparameters and for key hyperparameters, we set them as listed in Table~\ref{table6}. For other common hyperparameters with baselines, we reuse the default configurations of TimeGAN~\cite{yoon2019timegan} and VPSDE~\cite{song2021SDE} to conduct the regular time-series generation. We give our search space for the hyperparameters of TSGM. $iter_{pre}$ is in \{50000,100000\}. The dimension of hidden features, $d_{hidden}$, ranges from 2 times to 5 times the dimension of input features. On regular time-series generation, we follow the default values in TimeGAN~\cite{yoon2019timegan} and VPSDE~\cite{song2021SDE}. For irregular time-series tasks, we search the hidden dimension of decoder from 2 times to 4 times the dimension of input dimension, and follow GTGAN~\cite{https://doi.org/10.48550/arxiv.2210.02040} for other settings of NCDE-encoder and GRU-ODE-decoder.  

We follow default values for miscellaneous settings which are not explained on the baseline papers. Additionally, to deal with irregular time-series, we search the hyperparameters of GRU-D, which substitutes for RNN or are added to the head of baselines. We test the hidden dimension of GRU-D from 2 times to 4 times the dimension of input features.

\begin{table}[h]
\small
\centering
\setlength{\tabcolsep}{10pt}
\renewcommand{\arraystretch}{1.0}

\subfloat[Regular settings\label{tab6:a}]{
\scalebox{0.8}{
\begin{tabular}{c|cccc}
\toprule
\multirow{2}{*}{Dataset} & \multicolumn{4}{c}{Regular Settings} \\
 & $\dim(\mathbf{h})$ & $use_{alt}$ & $iter_{pre}$ & $iter_{main}$ \\
\midrule
Stocks & 24 & True  & 50000  & \multirow{4}{*}{40000} \\
Energy & 56 & False & 100000 & \\
Air    & 40 & True  & 50000  & \\
AI4I   & 24 & True  & 50000  & \\
\bottomrule
\end{tabular}}
}

\subfloat[Irregular settings\label{tab6:b}]{
\scalebox{0.8}{
\begin{tabular}{c|ccccc}
\toprule
\multirow{2}{*}{Dataset} & \multicolumn{5}{c}{Irregular Settings} \\
 & $D_{hidden}$ & $\dim(\mathbf{h})$ & $use_{alt}$ & $iter_{pre}$ & $iter_{main}$ \\
\midrule
Stocks & 48  & 24 & True  & \multirow{4}{*}{50000} & \multirow{4}{*}{40000} \\
Energy & 112 & 56 & False &                        &                        \\
Air    & 40  & 40 & True  &                        &                        \\
AI4I   & 48  & 24 & True  &                        &                        \\
\bottomrule
\end{tabular}}
}

\caption{The best hyperparameter setting for TSGM on regular and irregular time-series. $D_{hidden}$ denotes the hidden dimension of the GRU-ODE decoder.}
\label{table6}
\vspace{-1.em}
\end{table}

The following software and hardware environments were used for all experiments: \textsc{Ubuntu} 18.04 LTS, \textsc{Python} 3.9.12, \textsc{CUDA} 9.1, \textsc{NVIDIA} Driver 470.141, i9 CPU, and \textsc{GeForce RTX 2080 Ti}.

\begin{table}[h]
\centering
\begin{tabular}{c|c|c|c}
    \hline
    Dataset & Dimension & \#Samples & Length \\
    \hline
    Stocks & 6 & 3685 & \multirow{4}{1em}{24}\\
    Energy & 28 & 19735 &\\
    Air & 13 & 9357 &\\
    AI4I & 5 & 10000 &\\
    \hline
\end{tabular}
\caption{Characteristics of the datasets we use for our experiments}
\label{table1}
\vspace{-2.em}
\end{table}

\begin{table*}[t]
\centering
\renewcommand{\arraystretch}{1.3}
\setlength{\tabcolsep}{2.5pt}
\scalebox{0.77}{
\begin{tabular}{c|c|c|c|c|c|c|c|c|c|c|c}
    \hline
    
        \multicolumn{2}{c|}{Method} & \multicolumn{2}{|c|}{TSGM} & \multicolumn{2}{|c|}{Depth of 3} & \multicolumn{2}{|c|}{500 steps} & \multicolumn{2}{|c|}{250 steps} & \multicolumn{2}{|c}{100 steps} \\ \hline
        \multicolumn{2}{c|}{SDE} & VP & subVP & VP & subVP & VP & subVP & VP & subVP & VP & subVP \\ \hline
        \multirow{2}{*}{\rotatebox{90}{Disc.}} & Stocks & .022±.005 & .021±.008 & .022±.004 & \textbf{.020}±\textbf{.007} & .025±.005 & .020±.004 & .067±.009 & .022±.009 & .202±.013 & .023±.005 \\ 
        & Energy & .221±.025 & .198±.025 & \textbf{.175}±\textbf{.009} & .182±.009 & .259±.003 & .248±.002 & .250±.003 & .247±.002 & .325±.003 & .237±.004 \\ \hline
        \multirow{2}{*}{\rotatebox{90}{Pred.}} & Stocks & .037±.000 & \textbf{.037}±\textbf{.000} & .037±.000 & .037±.000 & .037±.000 & .037±.000 & .037±.000 & .037±.000 & .039±.000 & .037±.000 \\ 
        & Energy & .257±.000 & \textbf{.252}±\textbf{.000} & .253±.000 & .253±.000 & .257±.000 & .253±.000 & .256±.000 & .253±.000 & .256±.000 & .253±.000 \\ \hline
        
\end{tabular}
}
\caption{Sensitivity results on the depth of $M_\theta$ and the number of sampling steps. Our default TSGM has a depth of 4 and its number of sampling steps is 1,000. For other omitted datasets, we observe similar patterns.}
\label{tbl:sens}
\vspace{-1.em}
\end{table*}

\subsection{Experimental Results}
At first, on the regular time-series generation, Table~\ref{table2_1} shows that our method achieves remarkable results, outperforming most of the other baselines. Especially, for Stock, Energy, and Air, TSGM exhibits overwhelming performance by large margins for the discriminative score.
Moreover, for the predictive score, TSGM performs the best and obtains almost the same scores as that of the original data, which indicates that generated samples from TSGM preserve all the predictive characteristics of the original data.

Next, on the irregular time-series generation, we give the result with the 30\% missing rate setting on Table~\ref{table2_1}. TSGM also defeats almost all baselines by large margins on both the discriminative and predictive scores. Interestingly, VP generates poorer data as the missing rate grows up, while subVP synthesizes better one. It is worth mentioning that Diffusion-TS, which performs well in regular settings, achieves inferior results in irregular settings, highlighting that its limitations are closely tied to regular time-series generation.

We show t-SNE visualizations and KDE plots for the regular time-series generation in Figure~\ref{fig2} and Figure~\ref{fig:kde}. TimeGAN, GT-GAN, and TimeVAE are representative GAN or VAE-based baselines. In the figures, unlike the baseline methods, the synthetic samples generated from TSGM consistently show successful recall from the original data. 

Furthermore, TSGM generates more diverse synthetic samples compared to the three representative baselines across all cases. Notably, TSGM achieves significantly higher diversity on the 
Energy and Air dataset, which exhibits the most complex correlations (cf. Fig.~\ref{fig2}). Diffusion models are known to produce more diverse data than GANs~\cite{bayat2023a,yang2023diffusion}. As demonstrated by our results, TSGM synthesizes diverse data, highlighting another advantage of using diffusion models as anticipated.

\subsection{Sensitivity and Ablation Studies}

We conduct two sensitivity studies on regular time-series: i) reducing the depth of our score network, ii) decreasing the sampling step numbers. The results are in Table~\ref{tbl:sens}.
At first, we modify the depth of our score network from 2 to 3 to check the sensitivity of a depth of score network. Surprisingly, we achieve a better discriminative score with a slight loss on the predictive score.
Next, we decrease the number of sampling steps for faster sampling from 1,000 steps to 500, 250, and 100 steps, respectively. For VP, the case of 500 steps achieves almost the same results as that of original TSGM. Surprisingly, in the case of subVP, we achieve good results until 100 steps.


\begin{table}
\vspace{-1.em}
\scriptsize
\centering
\begin{tabular}{c|c|c|c|c}
        \hline
        & Method & SDE & Stocks & Energy \\
        \cline{2-5}
        \parbox[t]{2mm}{\multirow{4}{*}{\rotatebox{90}{Disc.}}} & \multirow{2}{*}{TSGM} & VP & .022$\pm$.005 & .221$\pm$.025 \\
        & & subVP & \textbf{.021}$\pm$\textbf{.008} & \textbf{.198}$\pm$\textbf{.025} \\
        \cline{2-5}
        & {\multirow{2}{*}{w/o pre-training}} & VP & .022$\pm$.004 & .322$\pm$.003 \\
        & & subVP & .059$\pm$.006 & .284$\pm$.004  \\
        \hline
        \parbox[t]{2mm}{\multirow{4}{*}{\rotatebox{90}{Pred.}}} & \multirow{2}{*}{TSGM} & VP & .037$\pm$.000  & .257$\pm$.000 \\
        & & subVP & \textbf{.037}$\pm$\textbf{.000} & .252$\pm$.000 \\
        \cline{2-5}
        & \multirow{2}{*}{w/o pre-training} & VP & .037$\pm$.000 & .252$\pm$.000 \\
        & & subVP & .037$\pm$.000 & \textbf{.251}$\pm$\textbf{.000} \\
        \hline
\end{tabular}
\caption{Comparison between with and without pre-training the autoencoder}
\label{tbl:abl}
\vspace{-2.em}
\end{table}

\begin{table}[h]
\centering
\begin{tabular}{c|cc|cc}
    \toprule
    Method & Stock & Energy & Stock & Energy \\ \midrule
    TimeGAN & 1.1 & 1.6  & 0.43  & 0.47  \\
    GTGAN & 2.3 & 2.3  & 0.43  & 0.47 \\
    TSGM & 1.9 & 1.9  & 86.32 & 85.89 \\
    \bottomrule
\end{tabular}
\caption{The memory usage of TimeGAN, GTGAN, and TSGM for training, and the sampling time for generating 100 samples on each dataset. Note that the original score-based model requires 3,214 seconds to sample 1,000 CIFAR-10 images, while StyleGAN needs 0.4 seconds, which is similar to the comparison between TSGM and TimeGAN.}
\label{table7}
\vspace{-2.em}
\end{table}


As an ablation study, we simultaneously train the conditional score network, encoder, and decoder from scratch on regular time-series generation (i.e., without the pre-training process). The results are in Table~\ref{tbl:abl}. These ablation models are worse than the full model due to the increased training complexity, but they still outperform many baselines. This ablation study shows the efficacy of pre-training our autoencoder.

\vspace{-.5em}

\subsection{Empirical Space and Time Complexity Analyses}\label{emp_time}

We report the memory usage during training and the wall-clock time for generating 100 time-series samples in Table~\ref{table7}. We compare TSGM to TimeGAN~\cite{yoon2019timegan} and GTGAN~\cite{https://doi.org/10.48550/arxiv.2210.02040}.
Our method is relatively slower than TimeGAN and GTGAN, which is a fundamental drawback of all SGMs. For example, the original score-based model~\cite{song2021SDE} requires 3,214 seconds for sampling 1,000 CIFAR-10 images while StyleGAN~\cite{DBLP:journals/corr/abs-1912-04958} needs 0.4 seconds. However, we also emphasize that this problem can be relieved by using the techniques suggested in~\cite{xiao2022tackling,DBLP:journals/corr/abs-2105-14080} as we mentioned in the conclusion section.


\section{Conclusions}


We presented a score-based generative model framework for universal time-series generation. We combined an autoencoder and our score network into a single framework to accomplish the goal --- our framework supports RNN-based or continuous-time method-based autoencoders. We also designed an appropriate denoising score matching loss for our generation task and achieved state-of-the-art results on various datasets in terms of the discriminative and predictive scores. In addition, we conducted rigorous ablation and sensitivity studies to prove the efficacy of our model design.

\section*{Acknowledgements}

Noseong Park is the corresponding author of this paper. This work was supported by Samsung SDS. This work was also supported by Institute of Information \& communications Technology Planning \& Evaluation (IITP) grant funded by the Korea government (MSIT) (No.2022-0-00857, Development of Financial and Economic Digital Twin Platform based on AI and Data)

\bibliographystyle{IEEEtran}
\bibliography{bib/IEEEabrv,bib/references}

\newpage
\appendix
\onecolumn
\section{Proofs}\label{appen:proof}

We introduce an additional lemma to prove Theorem~\ref{thm1}. In the following lemma, we state about the denoising score matching on time-series.

\begin{lemma}\label{lemma1}
In $L_{1}$ loss function, $l_{1}(n,s)$ can be replaced by the following $l^\star_{2}(n,s)$:

\begin{equation*}
l^\star_{2}(n,s) = \mathbb{E}_{\textbf{\textsc{x}}_n^0}\mathbb{E}_{{\textbf{\textsc{x}}}_{1:n}^s}\left[\left\|M_{\theta}(s,{\textbf{\textsc{x}}}_{1:n}^s,\textbf{\textsc{x}}_{1:n-1}^0)-{\nabla}_{{\textbf{\textsc{x}}}_{1:n}^s}\log p({\textbf{\textsc{x}}}_{1:n}^s|\textbf{\textsc{x}}_{1:n}^0)\right\|_2^2 \right],
\end{equation*}where $\textbf{\textsc{x}}_n^0$ and ${\textbf{\textsc{x}}}_{1:n}^s$ are sampled from ${p}(\textbf{\textsc{x}}_n^0|\textbf{\textsc{x}}_{1:n-1}^0)$ and ${p}({\textbf{\textsc{x}}}_{1:n}^s|\textbf{\textsc{x}}_{1:n}^0)$. Therefore, we can use an alternative objective, $L_{2} = \mathbb{E}_{s}\mathbb{E}_{{\textbf{\textsc{x}}}_{1:N}}\left[\sum_{n=1}^{N}\lambda(s)l^\star_{2}(n,s)\right]$ instead of $L_{1}$.
\end{lemma}

\begin{proof}
At first, if $n=1$, it can be substituted with the naive denoising score loss by~\cite{vincent2011matching} since $\textbf{x}_0^0 = \textbf{0}$. 

Next, let us consider $n>1$. $l_{1}(n,s)$ can be decomposed as follows:
\begin{equation*}
\begin{split}
l_{1}(n,s) = -2\cdot\mathbb{E}_{{\textbf{x}}_{1:n}^s}\langle M_{\theta}(s,{\textbf{x}}_{1:n}^s,\textbf{x}_{1:n-1}^0),{\nabla}_{{\textbf{x}}_{1:n}^s}\log p({\textbf{x}}_{1:n}^s|\textbf{x}_{1:n-1}^0)\rangle  + \mathbb{E}_{{\textbf{x}}_{1:n}^s}\left[\left\|M_{\theta}(s,{\textbf{x}}_{1:n}^s,\textbf{x}_{1:n-1}^0)\right\|_2^2\right]+C_1  
\end{split}
\end{equation*}

Here, $C_1$ is a constant that does not depend on the parameter $\theta$, and $\langle\cdot,\cdot\rangle$ means the inner product. Then, the first part's expectation of the right-hand side can be expressed as follows:
\begin{equation*}
\footnotesize
\begin{split}
\mathbb{E}_{{\textbf{x}}_{1:n}^s}[\langle M_{\theta}(s,{\textbf{x}}_{1:n}^s,\textbf{x}_{1:n-1}^0),{\nabla}_{{\textbf{x}}_{1:n}^s}\log p({\textbf{x}}_{1:n}^s|\textbf{x}_{1:n-1}^0)\rangle ]
=\int_{{\textbf{x}}_{1:n}^s}\langle M_{\theta}(s,{\textbf{x}}_{1:n}^s,\textbf{x}_{1:n-1}^0),{\nabla}_{{\textbf{x}}_{1:n}^s}\log p({\textbf{x}}_{1:n}^s|\textbf{x}_{1:n-1}^0)\rangle\text{p}({\textbf{x}}_{1:n}^s|\textbf{x}_{1:n-1}^0)d{\textbf{x}}_{1:n}^s
\\
=\int_{{\textbf{x}}_{1:n}^s}\langle M_{\theta}(s,{\textbf{x}}_{1:n}^s,\textbf{x}_{1:n-1}^0),\frac{1}{\text{p}(\textbf{x}_{1:n-1}^0)}{\partial\text{p}({\textbf{x}}_{1:n}^s,\textbf{x}_{1:n-1}^0)\over\partial {\textbf{x}}_{1:n}^s}\rangle d{\textbf{x}}_{1:n}^s
=\int_{{\textbf{x}}_{n}^0}\int_{{\textbf{x}}_{1:n}^s}\langle M_{\theta}(s,{\textbf{x}}_{1:n}^s,\textbf{x}_{1:n-1}^0),\frac{1}{\text{p}(\textbf{x}_{1:n-1}^0)}{\partial\text{p}({\textbf{x}}_{1:n}^s,\textbf{x}_{1:n-1}^0,\textbf{x}_n^0)\over\partial {\textbf{x}}_{1:n}^s}\rangle d{\textbf{x}}_{1:n}^sd\textbf{x}_n^0
\\
=\int_{{\textbf{x}}_{n}^0}\int_{{\textbf{x}}_{1:n}^s}\langle M_{\theta}(s,{\textbf{x}}_{1:n}^s,\textbf{x}_{1:n-1}^0),{\partial\text{p}({\textbf{x}}_{1:n}^s|\textbf{x}_{1:n}^0))\over\partial {\textbf{x}}_{1:n}^s}\rangle\frac{\text{p}(\textbf{x}_{1:n-1}^0,\textbf{x}_n^0)}{\text{p}(\textbf{x}_{1:n-1}^0)}d{\textbf{x}}_{1:n}^sd\textbf{x}_n^0
=\int_{{\textbf{x}}_{n}^0}\int_{{\textbf{x}}_{1:n}^s}\langle M_{\theta}(s,{\textbf{x}}_{1:n}^s,\textbf{x}_{1:n-1}^0),{\partial\text{p}({\textbf{x}}_{1:n}^s|\textbf{x}_{1:n}^0)\over\partial {\textbf{x}}_{1:n}^s}\rangle\text{p}(\textbf{x}_n^0|\textbf{x}_{1:n-1}^0)d{\textbf{x}}_{1:n}^sd\textbf{x}_n^0
\\
=\mathbb{E}_{\textbf{x}_n^0}\left[\int_{{\textbf{x}}_{1:n}^s}\langle M_{\theta}(s,{\textbf{x}}_{1:n}^s,\textbf{x}_{1:n-1}^0),{\partial\text{p}({\textbf{x}}_{1:n}^s|\textbf{x}_{1:n}^0)\over\partial {\textbf{x}}_{1:n}^s}\rangle d{\textbf{x}}_{1:n}^s\right]
=\mathbb{E}_{\textbf{x}_n^0}\left[\int_{{\textbf{x}}_{1:n}^s}\langle M_{\theta}(s,{\textbf{x}}_{1:n}^s,\textbf{x}_{1:n-1}^0),{\nabla}_{{\textbf{x}}_{1:n}^s}\log p({\textbf{x}}_{1:n}^s|\textbf{x}_{1:n}^0)\rangle\text{p}({\textbf{x}}_{1:n}^s|\textbf{x}_{1:n}^0)d{\textbf{x}}_{1:n}^s\right]
\\
=\mathbb{E}_{\textbf{x}_n^0}\mathbb{E}_{{\textbf{x}}_{1:n}^s}[\langle M_{\theta}(s,{\textbf{x}}_{1:n}^s,\textbf{x}_{1:n-1}^0),{\nabla}_{{\textbf{x}}_{1:n}^s}\log p({\textbf{x}}_{1:n}^s|\textbf{x}_{1:n}^0)\rangle]
\end{split}
\end{equation*}



Similarly, the second part's expectation of the right-hand side can be rewritten as follows:
\begin{equation*}
\footnotesize
\begin{split}
\mathbb{E}_{{\textbf{x}}_{1:n}^s}[\left\|M_{\theta}(s,{\textbf{x}}_{1:n}^s,\textbf{x}_{1:n-1}^0)\right\|_2^2]
=\int_{{\textbf{x}}_{1:n}^s}\left\|M_{\theta}(s,{\textbf{x}}_{1:n}^s,\textbf{x}_{1:n-1}^0)\right\|_2^2\cdot\text{p}({\textbf{x}}_{1:n}^s|\textbf{x}_{1:n-1}^0)d{\textbf{x}}_{1:n}^s\\
=\int_{\textbf{x}_n^0}\int_{{\textbf{x}}_{1:n}^s}\left\|M_{\theta}(s,{\textbf{x}}_{1:n}^s,\textbf{x}_{1:n-1}^0)\right\|_2^2\cdot\frac{\text{p}({\textbf{x}}_{1:n}^s,\textbf{x}_{1:n-1}^0,\textbf{x}_n^0)}{\text{p}(\textbf{x}_{1:n-1}^0)}d{\textbf{x}}_{1:n}^sd\textbf{x}_n^0
=\int_{\textbf{x}_n^0}\int_{{\textbf{x}}_{1:n}^s}\left\|M_{\theta}(s,{\textbf{x}}_{1:n}^s,\textbf{x}_{1:n-1}^0)\right\|_2^2\cdot\text{p}({\textbf{x}}_{1:n}^s|\textbf{x}_{1:n}^0)\frac{\text{p}(\textbf{x}_{1:n-1}^0,\textbf{x}_n^0)}{\text{p}(\textbf{x}_{1:n-1}^0)}d{\textbf{x}}_{1:n}^sd\textbf{x}_n^0\\
=\int_{\textbf{x}_n^0}\int_{{\textbf{x}}_{1:n}^s}\left\|M_{\theta}(s,{\textbf{x}}_{1:n}^s,\textbf{x}_{1:n-1}^0)\right\|_2^2\cdot\text{p}({\textbf{x}}_{1:n}^s|\textbf{x}_{1:n}^0)\text{p}(\textbf{x}_n^0|\textbf{x}_{1:n-1}^0)d{\textbf{x}}_{1:n}^sd\textbf{x}_n^0
=\mathbb{E}_{{\textbf{x}}_{n}^0}\mathbb{E}_{{\textbf{x}}_{1:n}^s}[\left\|M_{\theta}(s,{\textbf{x}}_{1:n}^s,\textbf{x}_{1:n-1}^0)\right\|_2^2] \cr
\end{split}
\end{equation*}

Finally, by using above results, we can derive following result:
\begin{equation*}
\begin{split}
l_{1} = \mathbb{E}_{{\textbf{x}}_{n}^0}\mathbb{E}_{{\textbf{x}}_{1:n}^s}\left[\left\|M_{\theta}(s,{\textbf{x}}_{1:n}^s,\textbf{x}_{1:n-1}^0)\right\|_2^2\right]+C_1
-2\cdot\mathbb{E}_{{\textbf{x}}_{n}^0}\mathbb{E}_{{\textbf{x}}_{1:n}^s}\langle M_{\theta}(s,{\textbf{x}}_{1:n}^s,\textbf{x}_{1:n-1}^0),{\nabla}_{{\textbf{x}}_{1:n}^s}\log p({\textbf{x}}_{1:n}^s|\textbf{x}_{1:n}^0)\rangle\\
=\mathbb{E}_{\textbf{x}_n^0}\mathbb{E}_{{\textbf{x}}_{1:n}^s}\left[\left\|M_{\theta}(s,{\textbf{x}}_{1:n}^s,\textbf{x}_{1:n-1}^0)-{\nabla}_{{\textbf{x}}_{1:n}^s}\log p({\textbf{x}}_{1:n}^s|\textbf{x}_{1:n}^0)\right\|_2^2 \right]+C
\cr
\end{split}
\end{equation*}

$C$ is a constant that does not depend on the parameter $\theta$. 
\end{proof}

\setcounter{theorem}{0}
\begin{theorem}[Autoregressive denoising score matching] $l_{1}(n,s)$ can be replaced with the following $l_{2}(n,s)$
\begin{equation*}
L_{score} = \mathbb{E}_{s}\mathbb{E}_{\textbf{x}_{1:N}^0}\left[\sum_{n=1}^{N}\lambda(s)l_2(n, s) \right],
\end{equation*}where
\begin{equation*}
l_{2}(n,s) = \mathbb{E}_{{\textbf{\textsc{x}}}_{1:n}^s}\left[\left\|M_{\theta}(s,{\textbf{\textsc{x}}}_{1:n}^s,\textbf{\textsc{x}}_{1:n-1}^0)-{\nabla}_{{\textbf{\textsc{x}}}_{1:n}^s}\log p({\textbf{\textsc{x}}}_{1:n}^s|\textbf{\textsc{x}}_{1:n}^0)\right\|_2^2\right].
\end{equation*}
Then, $L_{1}=L_{score}$ is satisfied.
\end{theorem}

\textit{proof.} By Lemma~\ref{lemma1}, it suffices to show that $L_{2}=L_{score}$. Whereas one can use the law of total expectation, which means \textit{$E[X]=E[E[X|Y]]$ if X,Y are on an identical probability space} to show the above formula, we calculate directly. At first, let us simplify the expectation of the inner part with a symbol $f(\textbf{x}_{1:n}^0)$ for our computational convenience, i.e.,
$f(\textbf{x}_{1:n}^0)=\mathbb{E}_{s}\mathbb{E}_{{\textbf{x}}_{1:n}^s}\left[\lambda(s)\left\|M_{\theta}(s,{\textbf{x}}_{1:n}^s,\textbf{x}_{1:n-1}^0)-{\nabla}_{{\textbf{x}}_{1:n}^s}\log p({\textbf{x}}_{1:n}^s|\textbf{x}_{1:n}^0)\right\|_2^2\right]$. Then we have the following definition:
\begin{equation*}
L_{2} = \mathbb{E}_{s}\mathbb{E}_{\textbf{x}_{1:N}^0}\left[l_2^\star\right]
=\mathbb{E}_{\textbf{x}_{1:N}^0}\left[\sum_{n=1}^{N}\mathbb{E}_{{\textbf{x}}_{n}^0}[f(\textbf{x}_{1:n}^0)]\right]=\sum_{n=1}^{N}\mathbb{E}_{\textbf{x}_{1:N}^0}\mathbb{E}_{{\textbf{x}}_{n}^0}[f(\textbf{x}_{1:n}^0)]
\end{equation*}

At last, the expectation part can be further simplified as follows:
\begin{equation*}
\begin{split}
\mathbb{E}_{\textbf{x}_{1:N}^0}\mathbb{E}_{{\textbf{x}}_{n}^0}[f(\textbf{x}_{1:n}^0)]
=\int_{\textbf{x}_{1:N}^0}\int_{\textbf{x}_n^0}f(\textbf{x}_{1:n}^0)p(\textbf{x}_{n}^0|\textbf{x}_{1:n-1}^0)d\textbf{x}_n^0\cdot p(\textbf{x}_{1:n-1}^0)p(\textbf{x}_{n:N}^0|\textbf{x}_{1:n-1}^0)d\textbf{x}_{1:N}^0\\
=\int_{\textbf{x}_{1:N}^0}\int_{\textbf{x}_n^0}f(\textbf{x}_{1:n}^0)p(\textbf{x}_{1:n}^0)d\textbf{x}_n^0\cdot p(\textbf{x}_{n:N}^0|\textbf{x}_{1:n-1}^0)d\textbf{x}_{1:N}^0
=\int_{\textbf{x}_{n:N}^0}\left({\int_{\textbf{x}_{1:n}^0}}f(\textbf{x}_{1:n}^0)p(\textbf{x}_{1:n}^0)d\textbf{x}_{1:n}^0\right)p(\textbf{x}_{n:N}^0|\textbf{x}_{1:n-1}^0)d\textbf{x}_{n:N}^0\\
=\int_{\textbf{x}_{1:n}^0}f(\textbf{x}_{1:n}^0)p(\textbf{x}_{1:n}^0)d\textbf{x}_{1:n}^0
=\int_{\textbf{x}_{1:n}^0}\left(\int_{\textbf{x}_{n+1:N}^0}p(\textbf{x}_{n+1:N}^0|\textbf{x}_{1:n}^0)d\textbf{x}_{n+1:N}^0\right)f(\textbf{x}_{1:n}^0)p(\textbf{x}_{1:n}^0)d\textbf{x}_{1:n}^0\\
=\int_{\textbf{x}_{1:N}^0}f(\textbf{x}_{1:n}^0)p(\textbf{x}_{1:N}^0)d\textbf{x}_{1:N}^0
=\mathbb{E}_{\textbf{x}_{1:N}^0}[f(\textbf{x}_{1:n}^0)] \\
\end{split}
\end{equation*}


Since $\sum_{n=1}^N\mathbb{E}_{\textbf{x}_{1:N}^0}[f(\textbf{x}_{1:n}^0)]=\mathbb{E}_{\textbf{x}_{1:N}^0}[\sum_{n=1}^N f(\textbf{x}_{1:n}^0)]=L_{score}$, we prove the theorem. \hfill $\square$

\section{Existing Time-Series Diffusion Models}\label{appen:score_other}

There exist time-series diffusion models for forecasting and imputation~\cite{rasul2021timegrad,tashiro2021csdi}. However, our approach to time-series synthesis is technically distinct. While these models aim to generate fulfilled samples given partially known time-series, TSGM focuses on producing diverse samples. Furthermore, TSGM is designed to provide diverse regular time-series given regular or irregular data unlike time-series forecasting, which predicts future observations based on past data, and time-series imputation, which infers missing elements within a given time-series sample. This intuition is reflected in the experimental results in Section~\ref{appen:exist_score_inappp}. 

We refer to a detailed analysis of these experiments in Section~\ref{analy}.

\subsection{Diffusion Models for Time-series Forecasting and Imputation}\label{sec:difftime}

TimeGrad~\cite{rasul2021timegrad} is a diffusion-based method for time-series forecasting, and CSDI~\cite{tashiro2021csdi} is for time-series imputation.

In TimeGrad~\cite{rasul2021timegrad}, they used a diffusion model for forecasting future observations given past observations. On each sequential order $n \in \{2,...,N\}$ and diffusion step $s \in \{1,...,T\}$, they train a neural network $\boldsymbol{\epsilon}_{\theta}(\cdot,\mathbf{x}_{1:n-1},s)$ with a time-dependent diffusion coefficient $\bar{\alpha}_s$ by minimizing the following objective function:
\begin{equation*}
    \mathbb{E}_{\textbf{x}_{n}^0,\epsilon,s} [\left\|\epsilon - \boldsymbol{\epsilon}_{\theta}(\sqrt{\bar{\alpha}_s}\mathbf{x}_n^0+\sqrt{1-\bar{\alpha}_s}\boldsymbol{\epsilon},\mathbf{x}_{1:n-1},s)\right\|_2^2],
\end{equation*}where $\boldsymbol{\epsilon} \sim \mathcal{N}(\boldsymbol{0}, \boldsymbol{I})$. The above formula assumes that we already know $\textbf{x}_{1:n-1}$, and by using an RNN encoder, $\textbf{x}_{1:n-1}$ can be encoded into $\textbf{h}_{n-1}$. After training, the model forecasts future observations recursively. More precisely speaking, $\textbf{x}_{1:n-1}$ is encoded into $\textbf{h}_{n-1}$ and the next observation $\textbf{x}_{n}$ is forecast from the previous condition $\textbf{h}_{n-1}$.

CSDI~\cite{tashiro2021csdi} proposed a general diffusion framework which can be applied mainly to time-series imputation. CSDI reconstructs an entire sequence at once, not recursively. Let $\textbf{x}^0 \in \mathbb{R}^{\dim(\textbf{X}) \times N}$ be an entire time-series sequence with $N$ observations in a matrix form. They define $\textbf{x}_{co}^0$ and $\textbf{x}_{ta}^0$ as conditions and imputation targets which are derived from $\textbf{x}^0$, respectively. They then train a neural network $\epsilon_{\theta}(\cdot,\textbf{x}_{co}^0,s)$ with a corresponding diffusion coefficient $\bar{\alpha}_s$ and a diffusion step $s \in \{1,...,T\}$ by minimizing the following objective function:
\begin{equation*}
    \mathbb{E}_{\textbf{x}^0,\epsilon,s}[\left\|\epsilon-\epsilon_{\theta}(s,\textbf{x}_{ta}^s,\textbf{x}_{co}^0)\right\|_2^2],
\end{equation*}
\noindent where $\textbf{x}_{ta}^s=\sqrt{\bar{\alpha}_s}\textbf{x}_{ta}^0+(1-\bar{\alpha}_s)\epsilon$. By training the network using the above loss, it generates missing elements from the partially filled matrix $\textbf{x}_{co}^0$. 


\subsection{Difference between Existing and Our Works}
Although they have earned state-of-the-art results for forecasting and imputation, we found that they are not suitable for our generative task due to the fundamental mismatch between their model designs and our task (cf. Table~\ref{table4} and Fig.~\ref{fig3}). 

\begin{table}[h]
\centering
\setlength{\tabcolsep}{10pt}
\renewcommand{\arraystretch}{1.0}
\begin{tabular}{c|c|c}
    \hline
    Method & Type  & Task Description \\
    \hline
    TimeGrad & Diffusion & From $\textbf{x}_{1:N-K}$, infer $\hat{\textbf{x}}_{N-K+1:N}$.\\
    CSDI & Diffusion & Given known values $\textbf{x}_{co}$, infer missing values $\hat{\textbf{x}}_{ta}$.\\
    \hline
    TimeGAN & GAN & Synthesize $\hat{\textbf{x}}_{1:N}$ from scratch.\\
    GT-GAN & GAN & Synthesize $\hat{\textbf{x}}_{1:N}$ from scratch.\\
    \hline
    TSGM & SGM & Synthesize $\hat{\textbf{x}}_{1:N}$ from scratch.\\
    \hline
\end{tabular}
\caption{Comparison among various recent GAN, diffusion, and SGM-based methods for time-series. $\textbf{x}_t$ (resp. $\hat{\textbf{x}}_t$) means a raw (resp. synthesized) observation at time $t$. For CSDI, $\textbf{x}_{co}$ means a set of known values and $\textbf{x}_{ta}$ means a set of target missing values --- it is not necessary that $\textbf{x}_{co}$ precedes $\textbf{x}_{ta}$ in time in CSDI.}
\label{table4}
\end{table}

\begin{figure}[t]
\centering
\includegraphics[width=1.0\columnwidth]{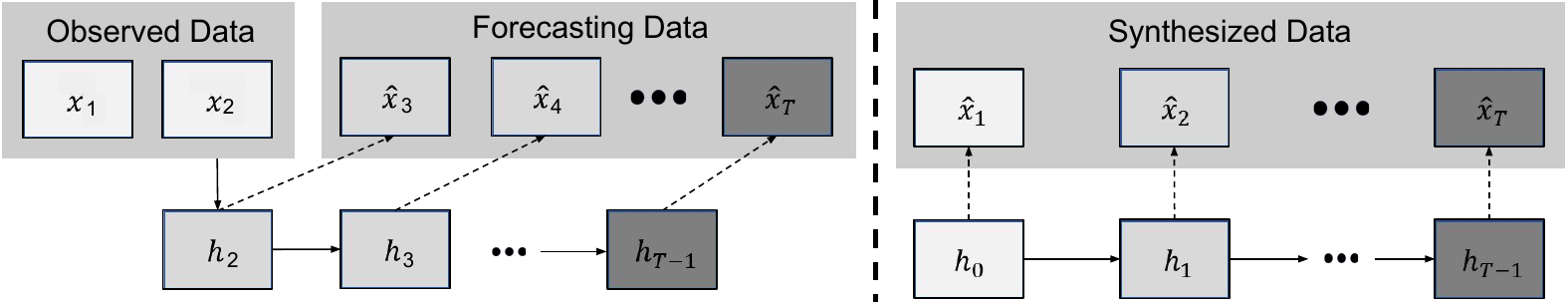} 
\caption{Graphical representation of TimeGrad (left) and TSGM (right). We adapt TimeGrad to our generation task but its results are not comparable even to other baselines' results (see Appendix~\ref{appen:adapt}).}
\label{fig3}
\end{figure}

TimeGrad generates future observations given the hidden representation of past observations $\textbf{h}_{n-1}$, i.e., a typical forecasting problem. Since our task is to synthesize from scratch, past known observations are not available. Thus, TimeGrad cannot be directly applied to our task. 


In CSDI, there are no fixed temporal dependencies between $\textbf{x}_{co}^0$ and $\textbf{x}_{ta}^0$ since its task is to impute missing values, i.e., $\textbf{x}_{ta}^0$, from known values, i.e., $\textbf{x}_{co}^0$, in the matrix $\textbf{x}^0$. It is not necessary that $\textbf{x}_{co}^0$ precedes $\textbf{x}_{ta}^0$ in time, according to the CSDI's method design. Our synthesis task can be considered as $\textbf{x}_{co}^0 = \emptyset$, which is the most extreme case of the CSDI's task. Therefore, it is not suitable to be used for our task.

To our knowledge, we are the first proposing an SGM-based time-series synthesis method. We propose to train a conditional score network by using the denoising score matching loss proposed by us, which is denoted as $L_{score}^{\mathcal{H}}$. Unlike other methods~\cite{rasul2021timegrad,tashiro2021csdi} that resort to existing known proofs, we design our denoising score matching loss in Eq.~\eqref{eq:ourscore} and prove its correctness. Meanwhile, TimeGrad and CSDI can be somehow modified for time-series synthesis but their generation quality is mediocre (see Appendix~\ref{appen:exist_score_inappp}).

 

\section{Experimental Results for Inapplicability of Existing Time-Series Diffusion Models to Our Work}\label{appen:exist_score_inappp}

In this section, we provide experimental results to show inapplicability of the existing time-series diffusion models, TimeGrad and CSDI, to the time-series generation task.

\subsection{Adapting TimeGrad toward Generation Task}\label{appen:adapt}
In this section, TimeGrad~\cite{rasul2021timegrad} is modified for the generation task. We simply add an artificial zero vector $\mathbf{0}$ in front of the all time-series samples of Energy. Therefore, TimeGrad's task becomes given a zero vector, forecasting (or generating) all other remaining observations. For the stochastic nature of its forecasting process, it can somehow generate various next observations given the sample input $\mathbf{0}$. Table~\ref{tbl:adapt} shows the experimental comparison between modified TimeGrad and TSGM in Energy for its regular time-series setting. TSGM gives outstanding performance, compared to modified TimeGrad. When checked in Table~\ref{table2_1}, modified TimeGrad is even worse than some baselines. Therefore, unlike TSGM, TimeGrad is not appropriate for the generation task.

\begin{table}
\centering
\begin{tabular}{c|c|c}
    \hline
    Method & Disc. & Pred. \\
    \hline
     TSGM-VP    & .221±.025 & .257±.000 \\
     TSGM-subVP & .198±.025 & .252±.000 \\
    
    \hline
    
     Modified TimeGrad  & .500±.000 & .287±.003 \\
    
    \hline
\end{tabular}
\caption{ Comparison between TSGM and modified TimeGrad in Energy for its regular time-series setting}\label{tbl:adapt}
\end{table}

\subsection{Adapting CSDI toward Generation Task} \label{sec:oneshot}

In this section, we apply CSDI to the unconditional time-series generation task by regarding all observations as missing values (i.e., $\textbf{x}_{co}^0=\mathbf{0}$) varying i) its kernel size from 1 to 7 and ii) the number of diffusion steps from 50 to 250. However, as demonstrated in Table~\ref{tbl:csdi}, CSDI fails to generate reliable time series samples in the datasets for its regular time series setting. In particular, TSGM with 250 steps in the ablation study section significantly outperforms it. Hence, we conclude that CSDI's unconditional generation is unsuitable for the time-series generation task.

\subsection{Analysis on Experimental Results}\label{analy} Until now, we have demonstrated the inferior results of forecasting and imputation on time-series generation. As shown by the results, these two models achieve relatively good predictive score but poor discriminative score, indicating a lack of diversity. This is because the primary goal of imputation and forecasting is to generate precise values that closely match the ground truth, as we discussed in Section~\ref{appen:score_other}. For example, CSDI takes its imputed time-series by averaging synthesized samples 100 times. Therefore, as the quality of forecasting and imputation improves, the diversity of the generated samples decreases, which means they are not suitable for time-series generation.

\begin{table}
\centering
\renewcommand{\arraystretch}{1.0}
\setlength{\tabcolsep}{4pt}
\small
\scalebox{0.9}{
    \begin{tabular}{c|cc|cc|cc|cc}
    \toprule
        \multirow{2}{*}{Method} & \multicolumn{2}{c|}{Stock} & \multicolumn{2}{c|}{Air} & \multicolumn{2}{c|}{Energy} & \multicolumn{2}{c}{AI4I} \\ 
        ~ & Disc. & Pred. & Disc. & Pred. & Disc. & Pred. & Disc. & Pred. \\  \midrule
        TSGM-VP & .022$\pm$.005 & \textbf{.037$\pm$.000} & \textbf{.122$\pm$.014} & \textbf{.005$\pm$.000} & \textbf{.147$\pm$.005} & \textbf{.217$\pm$.000} & .221$\pm$.025 & .257$\pm$.000 \\ 
        TSGM-subVP & \textbf{.021$\pm$.008} & \textbf{.037$\pm$.000} & .127$\pm$.010 & \textbf{.005$\pm$.000} & .150$\pm$.010 & \textbf{.217$\pm$.000} & \textbf{.198$\pm$.025} & .252$\pm$.000 \\ 
        \midrule
        Modified CSDI & .379$\pm$.008 & .045$\pm$.001 & .437$\pm$.144 & .040$\pm$.001 & .427$\pm$.081 & .217$\pm$.000 & .500$\pm$.000 & \textbf{.251$\pm$.000} \\ \bottomrule
    \end{tabular}}
\caption{Comparison between TSGM and modified CSDI in Stock, Air, Energy and AI4I for its regular time-series setting}
\label{tbl:csdi}
\end{table}

\section{Detailed Description of GRU-D}\label{app:grud}
GRU-D~\cite{Che2016RecurrentNN} is a modified GRU model which is for learning time-series data with missing values. This concept is similar with our problem statement, so we apply it to our baseline for irregular case. GRU-D needs to learn decaying rates along with the values of GRU.
First, GRU-D learns decay rates which depict vagueness of data as time passed. After calculating the decay rates, each value is composed of decay rate, mask, latest observed data, and predicted empirical mean that of GRU. The code can be utilized in the following link: https://github.com/zhiyongc/GRU-D

\end{document}